\documentclass{article}

\usepackage{microtype}
\usepackage{graphicx}
\usepackage{booktabs} 
\usepackage{environ}
\usepackage{subcaption}
\usepackage[CJKbookmarks=true,
            bookmarksnumbered=true,
            bookmarksopen=true,
            colorlinks=true,
            citecolor=blue,
            linkcolor=red,
            anchorcolor=red,
            urlcolor=blue
            ]{hyperref}
\usepackage{geometry}
\usepackage{algorithm}
\usepackage{algorithmic}

\usepackage{amsmath}
\usepackage{amssymb}
\usepackage{mathtools}
\usepackage{amsthm}

\usepackage[capitalize,noabbrev]{cleveref}

\theoremstyle{plain}
\newtheorem{theorem}{Theorem}[section]

\newtheorem{corollary}[theorem]{Corollary}
\theoremstyle{definition}

\theoremstyle{remark}
\newtheorem{remark}[theorem]{Remark}
\usepackage{authblk}

\usepackage[textsize=tiny]{todonotes}

\usepackage{babel}
\usepackage{verbatim}
\usepackage{mathtools}
\usepackage{bm}
\usepackage{natbib}
\usepackage{amsmath}
\usepackage{amssymb}
\usepackage{multicol}
\setcitestyle{authoryear,open={(},close={)}} 

\newcommand{\numobs}{\ensuremath{n}}

\input{macro2.tex}

\usepackage{bm, bbm}
\usepackage{enumitem}
\usepackage{multirow}
 
\theoremstyle{plain}

\usepackage{caption}
\usepackage{subcaption}
\theoremstyle{definition}

\usepackage{color}
\definecolor{cm}{RGB}{0,0,200}
\definecolor{purple}{RGB}{200,0,200}

\NewEnviron{resizealign}{\sbox0{
    $\begin{matrix}\displaystyle\BODY\end{matrix}$}%
  \sbox1{$(\theequation)$}%
  \sbox2{\parbox{\dimexpr \wd0 + 2\wd1}%
    {\begin{align}\BODY\end{align}}}
  \noindent\resizebox{\hsize}{!}{\usebox2}%
}

\title{Iterative Data Smoothing: Mitigating Reward Overfitting and Overoptimization in RLHF}
\date{}
\author[1]{Banghua Zhu}
\author[1]{Michael I. Jordan}
\author[1]{Jiantao Jiao}

\affil[1]{University of California, Berkeley}
\begin{document}

\maketitle


 

  
\begin{abstract}
Reinforcement Learning from Human Feedback (RLHF) is a pivotal technique that aligns language models closely with human-centric values. The initial phase of RLHF involves learning human values using a reward model from ranking data. It is  observed  that the performance of the reward model degrades after one  epoch of training, and optimizing too much against the learned reward model
eventually hinders the true objective. This paper delves into these issues, leveraging the theoretical insights to design improved reward learning algorithm termed 'Iterative Data Smoothing' (IDS). The core idea is that during each training epoch, we not only update the model with the data, but also update the date using the model, replacing hard labels with soft labels.   Our empirical findings highlight the superior performance of this approach over the traditional methods.

\end{abstract}

\section{Introduction}
Recent progress on Large Language Models (LLMs) is having a transformative effect not only in natural language processing but also more broadly in a range of AI applications~\citep{radford2019language, chowdhery2022palm,brown2020language,touvron2023llama,bubeck2023sparks, schulman2022chatgpt, openai2023gpt4, anthropic2023claude2}.  A key ingredient in the roll-out of LLMs is the fine-tuning step, in which the models are brought into closer alignment with specific behavioral and normative goals.  When no adequately fine-tuned, LLMs may exhibit undesirable and unpredictable behavior, including the fabrication of facts or the generation of biased and toxic content~\citep{perez2022red,ganguli2022red}. The current approach towards mitigating such problems is to make use of reinforcement learning based on human assessments.  In particular, Reinforcement Learning with Human Feedback (RLHF) proposes to use human assessments as a reward function from pairwise or multi-wise comparisons of model responses, and then fine-tune the language model based on the learned reward functions~\citep{ziegler2019fine, ouyang2022training, schulman2022chatgpt}. 

Following on from a supervised learning stage, a typical RLHF protocol involves two main steps:
\begin{itemize} 
    \item \textbf{Reward learning:} Sample prompts from a prompt dataset and generate multiple responses for the same prompt. Collect human preference data in the form of pairwise or multi-wise comparisons of different responses. Train a reward model based on the preference data. 
    \item \textbf{Policy learning:}  Fine-tune the current LLM based on the learned reward model with reinforcement learning algorithms.
\end{itemize}

Although RLHF has been successful in practice~\citep{bai2022training, ouyang2022training, dubois2023alpacafarm}, it is not without flaws, and indeed the current reward training paradigm grapples with  significant value-reward mismatches. There are two major issues with the current paradigm:
\begin{itemize}
    \item \textbf{Reward overfitting:} During the training of the reward model, it has been observed that the test cross-entropy loss of the reward model can deteriorate after one epoch of training~\cite{ouyang2022training}.
    \item \textbf{Reward overoptimization:} When  training the policy model to maximize the reward predicted by the learned model, it has been observed that the ground-truth reward can increase when the policy is close in KL divergence to the initial policy, but decrease with continued training~\citep{gao2023scaling}.
\end{itemize}

In this paper, we investigate these issues in depth. We  simplify the formulation of RLHF to a multi-armed bandit problem and  reproduce the overfitting and overoptimization phenomena. We leverage theoretical insights in the bandit setting to design new algorithms that work well under practical fine-tuning scenarios.

\subsection{Main Results}

As our first contribution, we pinpoint the root cause of both reward overfitting and overoptimization.  We show that it is the \textbf{inadequacy of the cross-entropy loss for long-tailed preference datasets}. As illustrated in Figure~\ref{fig:illustration}, even a simple 3-armed bandit problem can succumb to overfitting and overoptimization when faced with such imbalanced datasets. Consider a scenario where we have three arms with true rewards given by  $r_1^\star = 1,  r_2^\star = r_3^\star = 0$, and the preference distribution is generated by the Bradley-Terry-Luce (BTL) model~\citep{bradley1952rank}, i.e. $\mathbb{P}(i\succ j ) = \exp(r_i^\star) / (\exp(r_i^\star) + \exp(r_j^\star))$. Suppose our preference dataset compares the first and second arms $1000$ times but only compares the first and third arm once, and let $n(i\succ j)$ denote the number of times that arm $i$ is preferred over arm $j$. The standar empirical cross-entropy loss used in the literature for learning the reward model~\citep{ouyang2022training, zhu2023principled} can be written as follows:
\begin{align*}
    - \sum_{i, j} n(i\succ j) \log\left(\frac{\exp(r_i)}{\exp(r_i)+\exp(r_j)}\right). 
\end{align*}
We know that the empirical values $n(1\succ 2)$ and $n(2\succ 1)$ concentrate around their means. However, we have  with probability $0.73$, $n(1\succ 3) = 1$ and $n(3\succ 1) = 0$, and with probability $0.27$, $n(1\succ 3) = 0$ and $n(3\succ 1) = 1$. In either case, the minimizer of the empirical entropy loss will satisfy either $\hat r_1 - \hat r_3 =-\infty$ or $\hat r_1 - \hat r_3 =+\infty$. This introduces a huge effective noise when the coverage is imbalanced. Moreover, the limiting preference distribution is very different from the ground truth distribution, leading to reward overfitting. Furthermore, since there is $0.27$ probability that $\hat r_1 - \hat r_3 = -\infty$, we will take arm $3$ as the optimal arm instead of arm $1$. This causes reward overoptimization during the stage of policy learning since the final policy converges to the wrong arm with reward zero.

\begin{figure}[!htbp]
    \centering
    \includegraphics[width=0.99\linewidth]{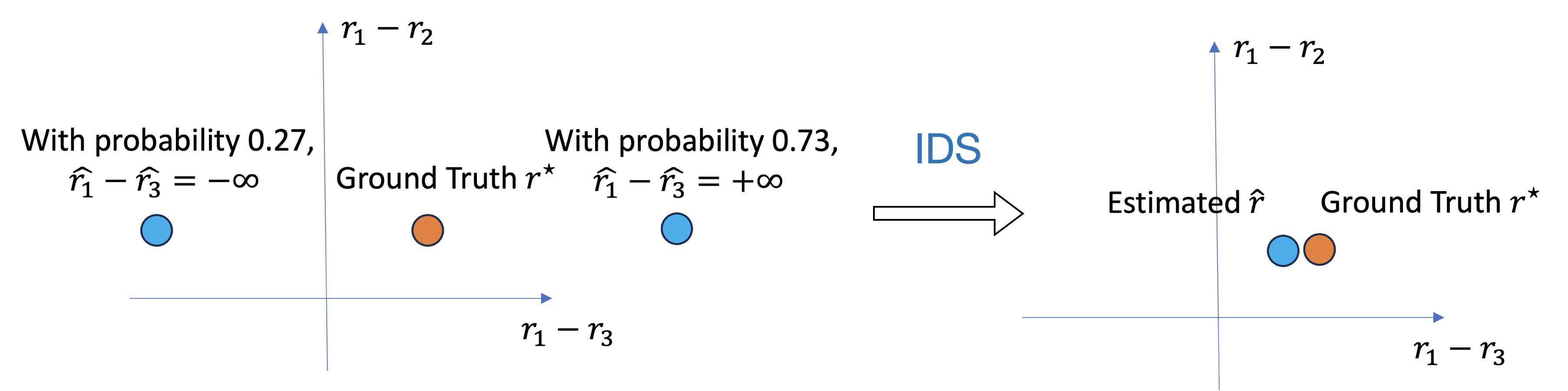}
    \caption{Illustration of the problem of the vanilla empirical cross-entropy minimization for learning the ground truth reward. With a small number of samples comparing arm $1$ and $3$, the minimization converges to a solution which assigns $\hat r_1 - \hat r_3 = -\infty$ with constant probability. With the proposed Iterative Data Smoothing (IDS) algorithm, the estimator is able to recover the ground truth reward.}
    \label{fig:illustration}
\end{figure}

To mitigate these effects, we leverage the pessimism mechanism from bandit learning to analyze and design a new algorithm, Iterative Data Smoothing (IDS), that simultaneously addresses both reward overfitting and reward overoptimization. The algorithm design is  straightforward: in each epoch, beyond updating the model with the data, we also adjust the data using the model.  Theoretically, we investigate the two phenomena in the tabular bandit case. We show that   the proposed method, as an alternative to the lower-confidence-bound-based algorithm~\citep{jin2021pessimism, xie2021bellman, rashidinejad2021bridging, zhu2023principled}, learns the ground truth distribution for pairs that are compared enough times, and ignores infrequently covered comparisons thereby mitigating issues introduced by long-tailed data. Empirically, we present experimental evidence that the proposed method improves reward training in both bandit and neural network settings.

\subsection{Related Work}
\paragraph{RLHF and Preference-based Reinforcement Learning.} RLHF, or Preference-based Reinforcement Learning (PbRL), has delivered significant empirical success in the fields of game playing, robot training, stock prediction, recommender systems, clinical trials and natural language processing~\citep{sui19,sadigh17,nipsPRL17,kupcsik18,jain13,wirth17, knox2008tamer, macglashan2017interactive, christiano2017deep, warnell2018deep, brown2019extrapolating, shin2023benchmarks, ziegler2019fine, stiennon2020learning, wu2021recursively, nakano2021webgpt,ouyang2022training, menick2022teaching, glaese2022improving, gao2022scaling, bai2022training, ganguli2022red, ramamurthy2022reinforcement}. In the setting of the language models, there has been work exploring the efficient fine-tuning of the policy model~\citep{snell2022offline, song2023preference, yuan2023rrhf, zhu2023fine, rafailov2023direct, wu2023pairwise}. 

In the case of reward learning,  \citet{ouyang2022training} notes that in general the reward can only be trained for one epoch in the RLHF pipeline, after which the test loss can go up.  
 \citet{gao2023scaling} studies the scaling law of training the reward model, and notes that overoptimization is another problem in reward learning.  To address the problem, \citet{zhu2023principled} propose a pessimism-based method that  improves the policy trained from the reward model when the optimal reward lies in a linear family. It is observed in~\citet{song2023reward} that the reward model tends to be identical regardless of whether the  prompts are open-ended or closed-ended during the terminal
phase of training, and they propose a prompt-dependent reward optimization scheme. 
 
Another closely related topic is the problem of estimation and ranking from pairwise or $K$-wise comparisons.  In the literature of \textit{dueling bandit},  one compares two actions and aims to minimize regret based on pairwise comparisons~\citep{Yue+12,Zoghi+14RCS, Yue+09,BTM,RDB,GS21,SG18, Ailon+14,Zoghi+14RUCB,Komiyama+15,Adv_DB,SGrank18,SGwin18, faury2020improved}.
\cite{sui19, xu20}  analyze the sample complexity of dueling RL agents in the tabular case, which is extended to linear case and function approximation by the recent work of~\citet{pacchiano2021dueling, chen2022human}. \citet{chatterji2022theory} studies a related setting where in each episode only binary feedback is received. Most of the theoretical work of learning from ranking focuses on  regret minimization,  while RLHF focuses more on the quality of the final policy.

\paragraph{Knowledge Distillation}
The literature of knowledge distillation focuses on transferring the knowledge from a teacher model to a student model \citep{hinton2015distilling, furlanello2018born, cho2019efficacy, zhao2022decoupled, romero2014fitnets, yim2017gift, huang2017like, park2019relational, tian2019contrastive, tung2019similarity, qiu2022better, cheng2020explaining}. It is observed in this literature that the soft labels produced by the teacher network can help train a better student network, even when the teacher and student network are of the same size and structure~\cite{hinton2015distilling}.  \citet{furlanello2018born} present a method which iteratively trains a new student  network after the teacher network achieves the smallest evaluation loss. Both our iterative data smoothing algorithm and these knowledge distillation methods learn from soft labels. However, iterative data smoothing iteratively updates the same model and data, while knowledge distillation method usually focuses on transferring knowledge from one model to the other.  

\section{Formulation}

We begin with the notation that we use in the paper. Then we introduce the general formulation of RLHF, along with our simplification in the multi-armed bandit case.

\textbf{Notations.} We use calligraphic letters for sets, e.g., $\mathcal{S}$ and $ \mathcal{A}$. Given a set $\mathcal{S}$, we write $|\mathcal{S}|$ to represent the cardinality of $\mathcal{S}$.   
We use $[K]$ to denote the set of integers from $1$ to $K$. We use $\mu(a, a')$ to denote the probability of comparing $a$ and $a'$ in a preference dataset, and $\mu(a) = \sum_{a'\in\mathcal{A}} \mu(a, a')$ to denote the probability of comparing $a$ with any other arms. Similarly, we use $n(a), n(a, a')$ to denote the number of samples that compare $a$ with any other arms, and the number of samples that compare $a$ with $a'$, respectively.  We  use $a_1\succ a_2$ to denote the event that the  $a_1$ is more preferred compared to $a_2$. 

\subsection{General Formulation of RLHF}
The key components in RLHF consist of two steps: reward learning and policy learning. We briefly introduce the general formulation of RLHF below.

In the stage of reward learning, one collects a preference dataset based on a prompt  dataset and responses to the prompts.  According to the formulation of \citet{zhu2023principled}, for the $i$-th  sample, a state (prompt) $s_i$ is first sampled from some prompt distribution $\rho$. Given the state $s_i$, $M$ actions (responses) $(a^{(1)}_i, a^{(2)}_i,\cdots, a^{(M)}_i)$ are sampled from some joint distribution $\mathbb{P}(a^{(1)},\cdots, a^{(M)} \mid s_i)$,
Let $\sigma_i:[M]\mapsto [M]$ be the output of the human labeller, which is a permutation function that denotes the ranking of  the actions. Here $\sigma_i(1)$ represents the most preferred action, and $\sigma_i(M)$ is the least preferred action. 
A common model for the distribution of $\sigma$ under multi-ary     comparisons is  a Plackett-Luce model~\citep{plackett1975analysis,luce2012individual}.
The Plackett-Luce model  defines the probability of  a state-action pair $(s, a_i)$ being the largest among a given set $\{(s, a_i)\}_{i=1}^{M}$ as
\begin{align*}
\mathbb{P}(a_i \succ a_j, \forall j\neq i \mid s)=
\frac{\exp(r^\star{(s, a_i)}) }{\sum_{j=1}^{M} \exp(r^\star{(s, a_j)})}, 
\end{align*}
where $r^\star:\mathcal{S}\times\mathcal{A}\mapsto \mathbb{R}$ is the ground-truth reward for the response given the prompt.
Moreover,  the probability of observing the permutation $\sigma$ is defined as\footnote{In practice, one may introduce an extra temperature parameter $\sigma$ and replace all $r^\star$ with $r^\star/\sigma$. Here  we take  $\sigma=1$.}
\begin{align*}
\mathbb{P}(\sigma\mid s, \{a^{(i)}\}_{i=1}^{M})=
\prod_{i=1}^{M}\frac{\exp(r(s, a^{(\sigma(i))}) )}{\sum_{j=i}^{M} \exp(r(s, a^{(\sigma(j))}))}.
\end{align*}

When $M=2$, this reduces to the pairwise comparison considered in the Bradley-Terry-Luce (BTL) model~\citep{bradley1952rank}, which is  used in existing RLHF algorithms. In this case, the permutation $\sigma$ can be reduced to a Bernoulli random variable, representing whether one action is preferred compared to the other. Concretely, for each queried state-actions pair $(s, a, a')$, we  observe a sample $c$ from a Bernoulli distribution with parameter $\frac{\exp(r_{\theta^\star}(s, a))}{\exp(r_{\theta^\star}(s, a))+\exp(r_{\theta^\star}(s, a'))}$.
Based on the observed dataset, the cross-entropy loss is minimized to estimate the ground-truth reward for the case of pairwise comparison. The minimizer of cross-entropy loss is the maximum likelihood estimator:
\begin{align}
\hat r_{\mathsf{MLE}} & \in \argmin_{r}    \mathcal{L}_{\mathsf{CE}}(\mathcal{D}, r),\nonumber \\
\mathcal{L}_{\mathsf{CE}}(\mathcal{D}, r) & = -\sum_{i=1}^n \log \Big(\frac{y_i\cdot \exp(r(s_i, a^{(1)}_i))}{\exp(r(s_i, a^{(1)}_i))+\exp(r(s_i, a^{(2)}_i))} +  \frac{(1-y_i)\cdot\exp(r(s_i, a^{(2)}_i))}{\exp(r(s_i, a_i^{(1)}))+\exp(r(s_i, a_i^{(2)}))}\Big). \nonumber
\end{align}

After learning the reward, we aim to learn the optimal policy under KL regularization with respect to an initial policy $\pi_0$ under some state (prompt) distribution $\rho'$. $$\hat\pi = \argmax_{\pi} \mathbb{E}_{s\sim \rho', a\sim \pi} [\hat r_{\mathsf{MLE}}(s, a)] - \lambda \cdot \mathbb{E}_{s\sim\rho'}[\mathsf{KL}(\pi(\cdot \mid s)\|\pi_0(\cdot \mid s))].$$
\subsection{RLHF in Multi-Armed Bandits }

To understand the overfitting and overoptimization problems, we simplify the RLHF problem to consider a single-state multi-armed bandit formulation with pairwise comparisons.  Instead of fitting a reward model and policy model with   neural networks, we fit a tabular reward model in a $K$-armed bandit problem. 

Consider a multi-armed bandit problem with $K$ arms. Each arm has a deterministic ground-truth reward   $r^\star(k) \in \mathbb{R}, k\in[K]$. In this case, the policy becomes a distribution supported on the $K$ arms $\pi\in\Delta([K])$. 
The sampling process for general RLHF reduces to the following: we first sample two actions $a_i$, $a_i'$ from a joint distribution $\mu \in \Delta([K]\times [K])$, and then observe a binary comparison variable $y_i$ following a distribution 
\begin{align*}
    \mathbb{P}(y_i = 1) = \frac{\exp(r^\star(a_i))}{\exp(r^\star(a_i)) + \exp(r^\star(a_i'))},  \quad \mathbb{P}(y_i = 0) = 1- \mathbb{P}(y_i = 1).  
\end{align*}

Assume that we are given $n$ samples, which are sampled  $i.i.d.$ from the above process. Let $n(a, a')$ be the total number of comparisons between actions $a$ and $a'$ in the $n$ samples. 
Let the resulting dataset be $\mathcal{D} = \{a_i, a_i', y_i\}_{i=1}^n$. The tasks in RLHF for multi-armed bandit setting can be simplified as:
\begin{enumerate}
    \item \textbf{Reward learning:} Estimate true reward $r^\star$ with a proxy reward $\hat r$ from the comparison dataset $\mathcal{D}$.  
    \item  \textbf{Policy learning:} Find a policy $\pi\in\Delta([K])$ that maximizes the proxy reward under KL constraints. 
\end{enumerate}
In the next two sections, we discuss separately the reward learning phase and policy learning phase, along with the reasons behind overfitting and overoptimization.

\subsection{Overfitting in Reward Learning}

For reward learning, 
the commonly used maximum likelihood estimator is the estimator that minimizes empirical cross-entropy loss: 
\begin{align}
\label{eq:MLE}
\hat r_{\mathsf{MLE}} & = \argmin_{r} \hat{\mathcal{L}}_{\mathsf{CE}}(\mathcal{D}, r), \text{ where }\\
\hat{\mathcal{L}}_{\mathsf{CE}}(\mathcal{D}, \hat r) & = -\frac{1}{n}
\sum_{i=1}^n   y_i \log\left(\frac{\exp(\hat r(a_i))}{\exp(\hat r(a_i)) + \exp(\hat r(a_i'))}\right)   + (1-y_i)  \log\left(\frac{\exp(\hat r(a_i'))}{\exp(\hat r(a_i)) + \exp(\hat r(a_i'))}\right). \nonumber
\end{align}

By definition, $\hat r_{\mathsf{MLE}}$ is convergent point when we optimize the empirical cross entropy fully. Thus the population cross-entropy loss of $\hat r_{\mathsf{MLE}}$ is an indicator for whether overfitting exists during reward training.

We  define the population cross entropy loss as 
\begin{align*}
    \mathcal{L}_{\mathsf{CE}}(r) & = -\mathbb{E}_{(a, a')\sim \mu, y\sim \mathsf{Ber}\left(\frac{\exp(r^\star(a))}{\exp(r^\star(a)) + \exp(r^\star(a'))}\right)} \Bigg[y\log\left(\frac{\exp(r(a))}{\exp(r(a)) + \exp(r(a'))}\right)   \\ 
    & \qquad + (1-y)\log\left(\frac{\exp( r(a'))}{\exp(r(a)) + \exp(\hat r(a'))}\right)\Bigg]  \\
    & = -\mathbb{E}_{(a, a')\sim \mu} \Bigg[\frac{\exp(r^\star(a))}{\exp(r^\star(a)) + \exp(r^\star(a'))}\log\left(\frac{\exp(r(a))}{\exp(r(a)) + \exp(r(a'))}\right)  \\
    & \qquad + \frac{\exp(r^\star(a'))}{\exp(r^\star(a)) + \exp(r^\star(a'))}\log\left(\frac{\exp( r(a'))}{\exp(r(a)) + \exp(\hat r(a'))}\right)\Bigg].
\end{align*}

For a fixed pairwise comparison distribution $\mu$, 
it is known that the maximum likelihood estimator $\hat r_{\mathsf{MLE}}$ converges to the ground truth reward $r^\star$ as the number of samples $n$ goes to infinity. 

\begin{theorem}[Consistency of MLE, see, e.g., Theorem 6.1.3. of \citet{hogg2013introduction}]\label{thm:mle}
    Fix $r^\star(K) = \hat r(K)=0$ for the uniqueness of the solution. For any fixed $\mu$, and any given ground-truth reward $r^\star$, we have that $\hat r_{\mathsf{MLE}}$ converges in probability to $r^\star$; i.e., for any $\epsilon>0$,
    \begin{align*}
   \lim_{n\rightarrow +\infty} \mathbb{P}\left( \|\hat r_{\mathsf{MLE}} - r^\star\|_\infty\geq \epsilon \right)  = 0. 
    \end{align*}
    Here we view $\hat r_{\mathsf{MLE}}$ and $r^\star$ as $K$-dimensional vectors.
\end{theorem}
The proof is deferred to Appendix~\ref{proof:mle_consistency}. 
This suggests that the overfitting phenomenon does not arise when we have an infinite number of samples. 
However, in the non-asymptotic regime when the comparison distribution $\mu$ may depend on $n$, one may not expect convergent result for MLE. We have the following theorem.

\begin{theorem}[Reward overfitting of MLE in the non-asymptotic regime]\label{thm:lower_mle}
    Fix $r^\star(a) = \mathbbm{1}(a=1)$ and $\hat r(K) = 0$ for uniqueness of the solution. For any fixed $n >  500$, there exists some 3-armed bandit problem such that with probability at least $0.09$, 
    \begin{align*}
    \mathcal{L}_{\mathsf{CE}}(\hat r_{\mathsf{MLE}}) - \mathcal{L}_{\mathsf{CE}}(r^\star)  \geq C
    \end{align*}
    for any arbitrarily large $C$.
\end{theorem}
The proof is deferred to Appendix~\ref{proof:lower_mle}. Below we provide a intuitive explanation. 
The constructed hard instance is a  bandit where $r^\star(a) = \mathbbm{1}(a=1)$. For any fixed $n$, we set $\mu(1, 2) = 1-1/n$,  $\mu(1, 3) = 1/n$.  

In this hard instance, there is constant probability that arm $3$ is only compared with $1$ once. And with constant probability, the observed comparison result between arm $1$ and arm $3$ will be different from the ground truth. The MLE will assign $r(3)=+\infty$ since the maximizer of $\log(\exp(x)/(1+\exp(x)))$ is infinity when $x$ is not bounded. Thus when optimizing the empirical cross entropy fully, the maximum likelihood estimator will result in a large population cross-entropy loss. We also validate this phenomenon in Section~\ref{sec:exp_bandit} with simulated experiments.

This lower bound instance simulates the high-dimensional regime where the number of samples is comparable to the dimension, and the data coverage is unbalanced across dimensions.  One can also extend the lower bound to more than 3 arms, where the probability of the loss being arbitrarily large  will be increased to close to 1 instead of a small constant.
\subsection{Overoptimization in Policy Learning}

After obtaining the estimated reward function $\hat r$, we optimize the policy $\pi\in\Delta([K])$ to maximize the estimated reward. In RLHF, one starts from an initial (reference) policy $\pi_0$, and optimizes the new policy $\pi$ to maximize the estimated  reward $\hat r$ under some constraint in KL divergence between $\pi$ and $\pi_0$. It is observed in~\cite{gao2022scaling} that as we continue optimizing the policy to maximize the estimated reward, the true reward of the policy will first increase then decrease, exhibiting the reward overoptimization phenomenon.

Consider the following policy optimization problem for a given reward model $\hat r$:
\begin{align}\label{eq:policy_learning}
    \max_{\pi\in\Delta([K])} \mathbb{E}_{a\sim \pi(\cdot)} [\hat r(a)] - \frac{1}{\lambda} \cdot \mathsf{KL}(\pi \| \pi_{0}).
\end{align}
Assuming that the policy gradient method converges to the optimal policy for the above policy optimization problem, which  has a closed-form solution: 
\begin{align}
    \pi_\lambda(a) = \frac{\pi_0(a)\cdot  \exp(\lambda\cdot \hat r(a))}{\sum_{a'\in\mathcal{A}} \pi_0(a')\cdot  \exp(\lambda \cdot \hat r(a'))} \label{eq:pi_lambda}.
\end{align}

In the tabular case, we can derive a closed form solution for how the KL divergence and ground-truth reward change with respect to $\lambda$, thus completely characterizing the reward-KL tradeoff. 
We compute the KL divergence and ground-truth reward of the policy as 
\begin{align*}
    \mathsf{KL}(\pi_\lambda \|\pi_0)  & =  \frac{\sum_{a\in\mathcal{A}} \pi_0(a)\cdot  \exp(\lambda\cdot \hat r(a)) \cdot \log(\exp(\lambda\cdot \hat r(a)) / (\sum_{a'\in\mathcal{A}} \pi_0(a')\cdot  \exp(\lambda\cdot\hat r(a'))))}{\sum_{a'\in\mathcal{A}} \pi_0(a')\cdot  \exp(\lambda\cdot\hat r(a'))} \\
    & =  \frac{\sum_{a\in\mathcal{A}} \pi_0(a)\cdot  \exp(\lambda\cdot\hat r(a)) \cdot \lambda\cdot\hat r(a)}{\sum_{a'\in\mathcal{A}} \pi_0(a')\cdot  \exp(\lambda\cdot\hat r(a'))} - \log\left(\sum_{a'\in\mathcal{A}} \pi_0(a')\cdot  \exp(\lambda\cdot\hat r(a'))\right), \\
    \mathbb{E}_{a\sim\pi_\lambda}[r^\star(a)] & = \frac{\sum_{a\in\mathcal{A}} \pi_0(a)\cdot  \exp(\lambda\cdot\hat r(a)) \cdot  \lambda\cdot r^\star(a)}{\sum_{a'\in\mathcal{A}} \pi_0(a')\cdot  \exp(\lambda\cdot\hat r(a'))}.
\end{align*}
The above equation provides a precise characterization of how the mismatch between $\hat r$ and $r^\star$ leads to the overoptimization phenomenon, which can be validated from the experiments in Section~\ref{sec:exp}. 
To simplify the analysis and provide better intuition, we focus on the case when $\lambda\rightarrow \infty$, i.e., when the optimal policy selects the best empirical arm without considering the KL constraint. In this case, the final policy reduces to the empirical best arm, $\pi_\infty(a) = \mathbbm{1}(a=\argmax_{a'} \hat r(a'))$.

By definition, $\pi_\infty$ is the convergent  policy when we keep loosening the KL divergence constraint in Equation (\ref{eq:policy_learning}). Thus the performance of  $\pi_\infty$ is a good indicator of whether overoptimization exists during policy training. We thus define a notion fo sub-optimality to characterize the performance gap between the convergent policy and the optimal policy:
\begin{align*}
\mathsf{SubOpt}(\hat\pi) \coloneqq
\max_a \mathbb{E}[r^\star(a) -r^\star( \hat \pi)].
\end{align*}

We know from Theorem~\ref{thm:mle} that, asymptotically, the MLE for reward $\hat r_{\mathsf{MLE}}$ converges to the ground truth reward $r^\star$. As a direct result, when using the MLE as reward, the sub-optimality of the policy $\pi_\infty$ also converges to zero with an infinite number of samples.

However, as a corollary of Theorem~\ref{thm:lower_mle} and a direct consequence of reward overfitting, $\pi_\infty$ may have large sub-optimality in the non-asymptotic regime when trained from $\hat r_{\mathsf{MLE}}$.  
\begin{corollary}[Reward overoptimization of MLE in the non-asymptotic regime]\label{thm:fail_mle}
    Fix $r^\star(a) = \mathbbm{1}(a=1)$. For any fixed $n$, there exists some 3-armed bandit problem such that with probability at least $0.09$, 
    \begin{align*}
    \mathsf{SubOpt}(\hat \pi_{\infty})\geq 1.
    \end{align*}
    \end{corollary}
The proof is deferred to Appendix~\ref{proof:fail_mle}. This suggests that $\hat r_{\mathsf{MLE}}$ also leads to the reward overoptimization phenomenon in the non-asymptotic regime. 
In  Section~\ref{sec:exp},  we conduct simulation in the exact same setting to verify the theoretical results.

\section{Methods:  Pessimistic MLE and Iterative Data Smoothing}

The problem of overfitting and overoptimization  calls for  a design of better and practical  reward learning algorithm that helps mitigate both issues. We first discuss the pessimistic MLE algorithm in~\cite{zhu2023principled}, which is shown to converge to a policy with vanishing sub-optimality under good coverage assumption.

\subsection{Pessimistic MLE}

 In the tabular case, the pessimistic MLE  corrects the original MLE by subtracting a confidence interval. Precisely, we have
 \begin{align}
     \hat r_{\mathsf{PE}}(a) = \hat r_{\mathsf{MLE}}(a) - \lambda\cdot \sqrt{\frac{1}{n}}, \label{eq:pessimism}
 \end{align}
where $n$ is the total number of samples and $\lambda = \|(L+\epsilon I )^{-1/2}_j\|_2$ is the norm of the $j$-th column of the matrix $(L+\epsilon I )^{-1/2}$, where $L$ is the matrix that satisfies   $L_{a, a} = n(a)/n$, $L_{a, a'} = -n(a, a')/n, \forall a\neq a'$, and $\epsilon$ is a small constant.  
Intuitively, for those arms that are compared fewer times, we are more uncertain about their ground-truth reward value. Pessimistic MLE penalizes these arms by directly subtracting the length of lower confidence interval of their reward, making sure that the arms that are less often compared will  be less likely to be chosen.  
It is shown in~\citet{zhu2023principled} that the sub-optimality of the policy optimizing $\hat r_{\mathsf{PE}}$ converges to zero under the following two conditions:
\begin{itemize}
    \item The expected number of times that one compares optimal arm (or the expert arm to be compared with in the definition of sub-optimality) is lower bounded by some positive constant $\mu(a^\star)\geq C$.
    \item The parameterized reward family lies in a bounded space $|\hat r(a) |\leq B, \forall a \in[K]$.
\end{itemize}

This indicates that  pessimistic MLE can help mitigate the reward overoptimization phenomenon. However, for real-world reward training paradigm, the neural network parameterized reward family may not be bounded. Furthermore, estimating the exact confidence interval for a neural-network parameterized model can be hard and costly. This prevents the practical use of pessimistic MLE, and calls for new methods that can potentially go beyond these conditions and apply to neural networks. 

  \subsection{Iterative Data Smoothing }
We propose a new algorithm, Iterative Data Smoothing (IDS), that shares similar insights as  pessimistic MLE. 
Intuitively,  pessimistic MLE helps mitigate the reward overoptimization issue by reducing the estimated reward for less seen arms. In IDS, we achieve this by updating the label of the data we train on.  

\begin{algorithm}[!htbp]
\caption{Iterative Data Smoothing ($\mathcal{D}, \theta_0, \alpha, \beta$)}
\label{alg:refine}
\begin{algorithmic}
\STATE \textbf{Input:} The pairwise comparison dataset $\mathcal{D} = \{a_i, a_i', y_i\}_{i=1}^n$. A parameterized reward model family $\{r_\theta: \mathcal{A} \mapsto \mathbb{R} \mid \theta\in\Theta\}$ with initialization $\theta_0\in\Theta$. Two step sizes $\alpha, \beta$. An empirical loss function $$\mathcal{L}_\theta(\{y_i\}, \mathcal{D})   =- \frac{1}{n}
\sum_{i=1}^n  y_i\cdot \log\left(\frac{\exp(  r_\theta(a_i))}{\exp( r_\theta(a_i)) + \exp( r_\theta(a_i'))}\right) + (1-y_i)\cdot \log\left(\frac{\exp(r_\theta(a_i'))}{\exp(r_\theta(a_i)) + \exp( r_\theta(a_i'))}\right) $$
\STATE Initialize $t=0$ and $y_{i, 0} = y_i, \forall i\in[n]$.
\WHILE{$r_{\theta_t}$ does not converge}{ \STATE \begin{align*}
        \theta_{t+1} & \gets \theta_t - \alpha \cdot \nabla \mathcal{L}_\theta(\{y_{i, t}\}, \mathcal{D})  \\
        y_{i, t+1} & \gets (1-\beta) \cdot y_{i, t} + \beta \cdot \frac{\exp(  r_{\theta_{t+1}}(a_i))}{\exp( r_{\theta_{t+1}}(a_i)) + \exp( r_{\theta_{t+1}}(a_i'))} \\
        t & \gets t+1
     \end{align*}
}
\ENDWHILE 
\STATE \textbf{Return:} $r_{\theta_{t}}$
\end{algorithmic}
\end{algorithm}

As is shown in Algorithm \ref{alg:refine}, we initialize $y_{i, 0}$ as the labels for the samples $y_i$. In the $t$-th epoch, we first update the model using the current comparison dataset with labels $\{y_{i, t}\}_{i=1}^n$. After the model is updated,  we also update the data using the model by predicting the probability  $\mathbb{P}(y_i = 1)$ for each comparison $(a_i, a_i')$ using the current reward estimate $\hat r_{\theta_t}$. We update each label $y_{i,t}$ as a weighted combination of its previous value and the new predicted probability. 

Intuitively, $y_{i, t}$ represents a proxy of the confidence level of labels predicted by interim model checkpoints. The idea is that as the model progresses through multiple epochs of training, it will bring larger change to rewards for frequently observed samples whose representation is covered well in the dataset. Meanwhile, for seldom-seen samples, the model will make minimal adjustments to the reward.

\subsubsection{Benefit of one-step gradient descent}

Before we analyze the IDS algorithm, we first discuss why training for one to two epochs in the traditional reward learning approach works well~\citep{ouyang2022training}.  We provide the following analysis of the one-step gradient update for the reward model. The proof is deferred to Appendix~\ref{proof:gd}.
\begin{theorem}\label{thm:gd}
    Consider the same multi-armed bandit setting where  the reward  is initialized equally for all $K$ arms. Then after one-step gradient descent, one has
    \begin{align*}
 \forall a, a'\in[K],  \hat r(a) - \hat r(a') = \alpha \cdot (n_+(a) - n_-(a) - (n_+(a') - n_-(a'))),
    \end{align*}
    where $n_+(a), n_-(a)$ refers to the total number of times that $a$ is preferred and not preferred, respectively.
\end{theorem} 
\begin{remark}
    The result shows that why early stopping in the traditional reward learning works well in a simple setting. After one gradient step, the empirical best arm becomes the  the arm whose absolute winning time is the largest. This can be viewed as another criterion besides pessimism that balances both the time of comparisons and the time of being chosen as the preferred arm. When the arm $a$ is only compared few times, the difference $n_+(a) - n_-(a)$ will be bounded by the total number of comparisons, which will be  smaller than those that have been compared much more times. Thus the reward model will penalize those arms seen less. After updating the label with the model prediction, the label of less seen samples will be closer to zero, thus getting implicitly penalized.
\end{remark}

\subsubsection{Benefit of iterative data smoothing}\label{sec:benefits_dr}
Due to under-optimization, the estimator from a one-step gradient update might still be far from  the ground-truth reward. 
We provide an analysis here  why IDS can be better. Consider any two arms $a, a'$ with $n(a, a')$ observations among $n$ total observations. By computing the gradient, we can write the IDS algorithm as
\begin{align*}
      \hat r_{t+1}(a) - \hat r_{t+1}(a') & =  \hat r_{t}(a) - \hat r_{t}(a') + \frac{\alpha\cdot n(a, a')}{n} \cdot \Big( ({\hat \mu(a\succ a')} \cdot y_t + \hat \mu(a \prec a') \cdot (1-y_t) ) \cdot  \frac{\exp(  \hat r_{ {t}}(a'))}{\exp( \hat r_{ {t}}(a)) + \exp( \hat r_{ {t}}(a'))} \\ 
    &  \qquad  - ({\hat \mu(a\prec a')} \cdot y_t + \hat \mu(a \succ a') \cdot (1-y_t) ) \cdot  \frac{\exp(\hat   r_{ {t}}(a))}{\exp( \hat r_{ {t}}(a)) + \exp( \hat r_{ {t}}(a'))} \Big)  \\
        y_{t+1} & =  (1-\beta) \cdot y_{t} + \beta \cdot \frac{\exp( \hat  r_{ {t+1}}(a))}{\exp( \hat  r_{ {t+1}}(a)) + \exp(\hat  r_{ {t+1}}(a'))},
     \end{align*}
where we define $\hat\mu(a\succ a') = n(a\succ a') / n(a, a')$. One can see that the effective step size for updating $\hat r$ is $\alpha \cdot n(a, a') /n$, while the effective step size for updating $y$ is $\beta$. Assume that we choose $\alpha, \beta, l, m$ such that
\begin{align*}
    \alpha\cdot l/n \ll \beta \ll \alpha\cdot m /n.
\end{align*}
Consider the following two scales:
\begin{itemize}
    \item When there are sufficient observations, $n(a, a')\geq m$, we know that $\beta \ll \alpha\cdot n(a, a') /n$. In this case, the update step size of $y_t$ is much slower than $\hat r_t$. One can approximately take $y_t \approx 0$ or $1$ as unchanged during the update. Furthermore, since $n(a, a')\geq m$ is large enough, $\hat \mu$ concentrates around the ground truth $\mu$.  In this case, one can see that the reward converges to the ground truth reward $ \hat r_t \rightarrow r^\star$.
    \item When the number of observations is not large, i.e., $n(a, a')\leq l$,  we know that $ \alpha\cdot l/n \ll \beta$. In this case,  the update of $\hat r$ is much slower than $y_t$. When the $\hat r_0$ are initialized to be zero, $y_t$ will first converge to $1/2$, leading to $\hat r_t(a)\approx \hat r_t(a')$ when $t$ is large.
\end{itemize}
To formalize the above argument, we consider the following differential equations:
\begin{align}
    \dot{d}(t) & = \alpha n \cdot \left((\mu\cdot y(t) + (1-\mu)\cdot (1-y(t)))\cdot \frac{1}{1+\exp(d(t))} - ((1-\mu)\cdot y(t) + \mu\cdot (1-y(t)))\cdot \frac{\exp(d(t))}{1+\exp(d(t))}\right) \nonumber \\ 
    \dot{y}(t) & =  \beta \cdot \left(\frac{\exp(d(t))}{1+\exp(d(t))} - y(t)\right) \label{eq:diff}.
\end{align}
Here $d$ represents the difference of reward between two arms $a, a'$, and $\mu$ represents the empirical frequency $\hat \mu(a\succ a')$. Let the initialization be $d(0)=0, y(0) = 1$.
We have the following theorem.
\begin{theorem}\label{thm:differential}
    The  differential equations  in Equation \eqref{eq:diff} have one unique stationary point $d(t) = 0, y(t) = \frac{1}{2}$. On the other hand, for any $\alpha, \beta, n, T$ with  $\beta T \leq \epsilon \ll 1 \ll \alpha n T$, one has
    \begin{align*}
     \left|\frac{\exp(d(T))}{1+\exp(d(T))} - \mu\right|  &  \leq \max(2(1-\exp(-\epsilon)),  \exp(-\mu(1-\mu)\alpha n T) ) \\ 
        y(T)& \geq \exp(-\epsilon).
    \end{align*}
\end{theorem}
The proof is deferred to Appendix~\ref{proof:differential}. Note that the above argument only proves convergence to the empirical measure $\mu$. One can combine standard concentration argument to prove the convergence to the ground truth probability. The result shows that when choosing $\alpha, \beta$ carefully, for the pair of arms with a large number of comparisons, the difference of reward  will be close to the ground truth during the process of training. As a concrete example, by taking $\alpha =  {n}^{-1/2}, \beta = n^{-1}T^{-2}, \epsilon = \beta T$, we have 
  \begin{align*}
     \left|\frac{\exp(d(T))}{1+\exp(d(T))} - \mu\right|  
     & \leq \max(2n^{-1}T^{-1},  \exp(-\mu(1-\mu)n^{1/2} T)). 
    \end{align*}

 One can see that for those pairs of comparisons with a large sample size $n$, the estimated probability is close to the ground truth probability. On the other hand, for those pairs  that are compared less often, the difference $d(t)$ is updated less frequently and remains close to the initialized values. Thus the algorithm implicitly penalizes the less frequently seen pairs, while still estimating the commonly seen pairs accurately.

In summary, the IDS algorithm enjoys several benefits:
\begin{itemize}
    \item For a sufficient number of observations, the estimated reward approximately converges to the ground truth reward; while for an insufficient number of observations, the estimated reward remains largely unchanged at the initialization. Thus the reward model penalizes the less observed arms with higher uncertainty.  
    \item It is easy to combine with neural networks, allowing arbitrary parametrization of the reward model.
    \item It utilizes the soft labels starting from the second epoch, which can be more effective than hard labels according to the literature on knowledge distillation~\citep{hinton2015distilling, zhao2023towards}. 
\end{itemize}

We also present an alternative formulation of IDS in Appendix~\ref{app:alternative}.

  \begin{figure}[!htbp]
     \centering
     \begin{subfigure}[b]{0.45\linewidth}
         \centering
         \includegraphics[width=\textwidth]{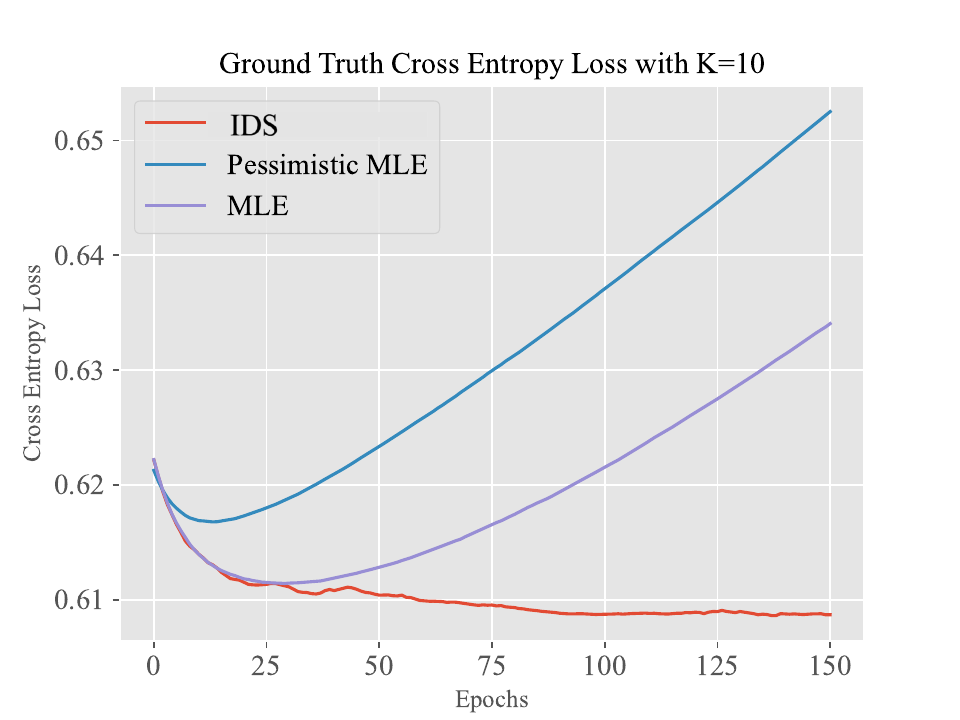}
     \end{subfigure}
     \hfill
     \begin{subfigure}[b]{0.45\linewidth}
         \centering
         \includegraphics[width=\textwidth]{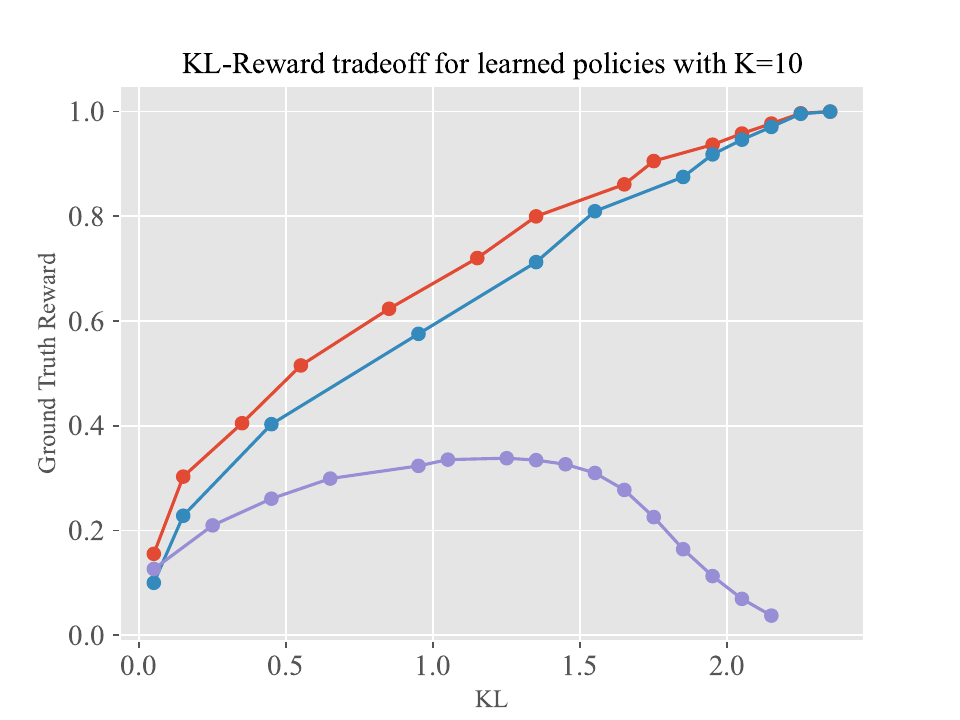}
     \end{subfigure}
     \begin{subfigure}[b]{0.45\linewidth}
         \centering
         \includegraphics[width=\textwidth]{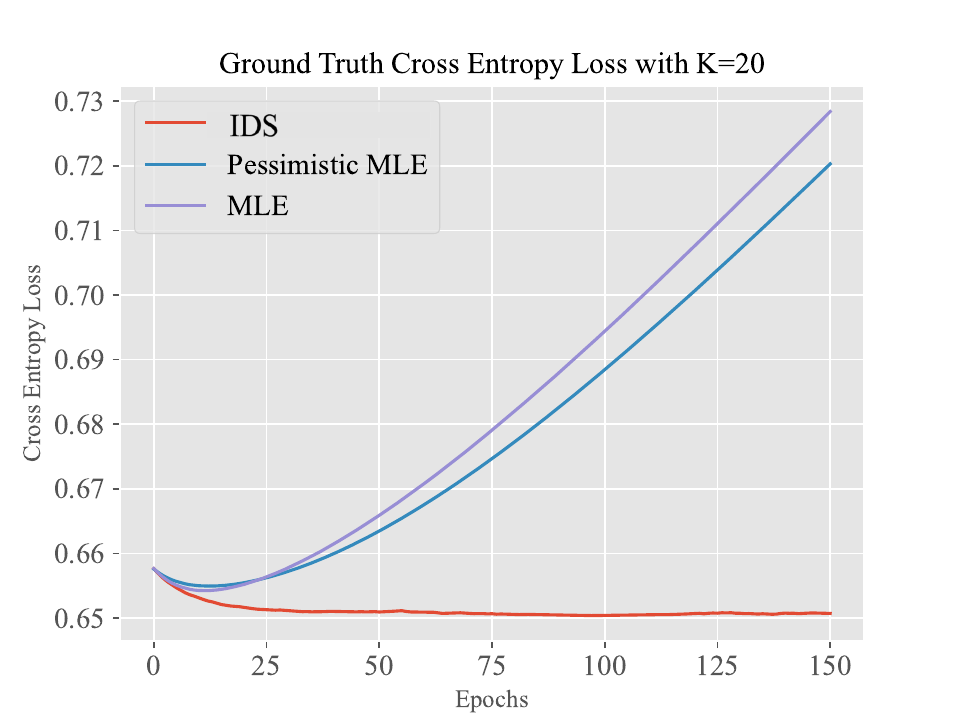}
     \end{subfigure}
     \hfill
     \begin{subfigure}[b]{0.45\linewidth}
         \centering
         \includegraphics[width=\textwidth]{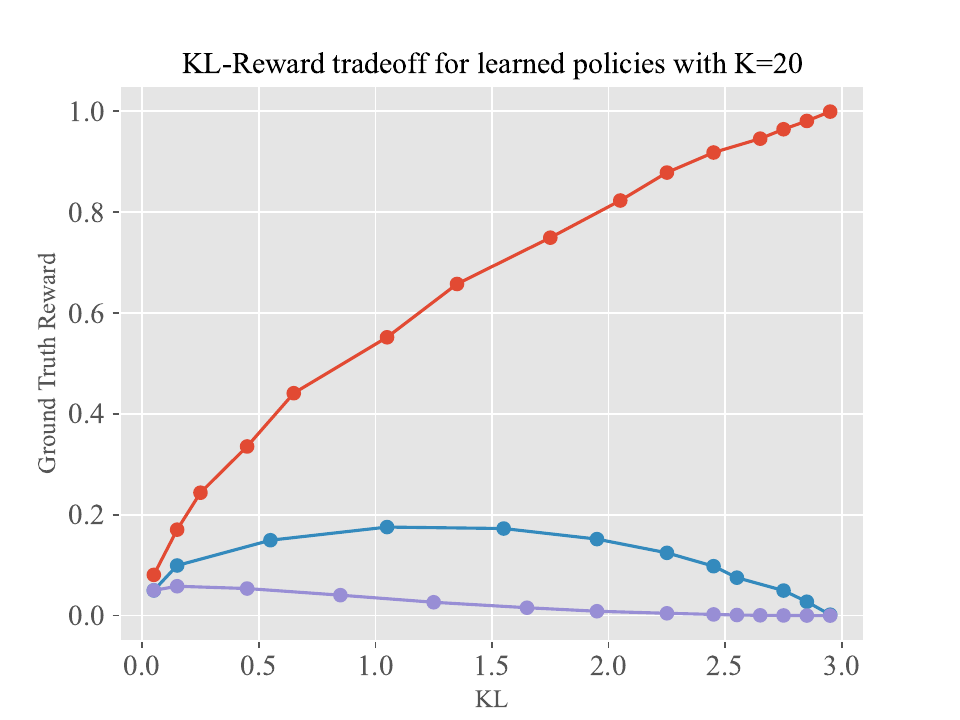}
     \end{subfigure}
    \caption{Comparisons of the three methods  in the multi-armed bandit setting. }
    \label{fig:bandit}
\end{figure}

 \section{Experiments}\label{sec:exp}

 In this section, we present the results of experiments with both  multi-armed bandits and neural networks.
\subsection{Multi-Armed Bandit}\label{sec:exp_bandit}

In the bandit setting, we focus on the hard example constructed in Theorem~\ref{thm:lower_mle}. We take total samples $n=60$ and the number of arms $K$ as $10$ and $20$. We compare the performance of the vanilla MLE, pessimistic MLE and IDS in both the reward learning phase and the policy learning phase. 

In the reward learning phase, we run stochastic gradient descent with learning rate $0.01$ on the reward model for multiple epochs and monitor how the loss changes with respect to the number of training epochs. For pessimistic MLE, we subtract the confidence level in the reward according to Equation (\ref{eq:pessimism}). For IDS, we take the two step sizes as $\alpha = 0.01, \beta = 0.001$.
As is shown in left part of Figure~\ref{fig:bandit}, both MLE and pessimistic MLE suffer from reward overfitting, while the test cross-entropy loss for  the IDS algorithm continues to decrease until convergence. Since the training loss changes with the updated labels, we   plot the population cross-entropy loss which is averaged over all pairs of comparisons.

In the right part of the figure, we plot the KL-reward tradeoff when training a policy based on the learned reward. We vary the choice of $\lambda$  in Equation (\ref{eq:pi_lambda}) to derive the optimal policy under diverse levels of KL constraint, where we take the reference policy $\pi_0$  as the uniform policy.  One can see that IDS is able to converge to the optimal reward when KL is large, while both MLE and pessimistic MLE suffer from overoptimization.

We remark here that the reason  pessimistic MLE suffers from both overfitting and overoptimization might be due to the design of unbounded reward in the multi-armed bandit case. When the reward family is bounded, pessimistic MLE is also guaranteed to mitigate the overoptimization issue. Furthermore, we only run one random seed for this setting to keep the plot clean since the KL-reward trade-off heavily depends on the observed samples. 

\begin{figure}[htbp]
     \centering
     \begin{subfigure}[b]{0.49\linewidth}
         \centering
         \includegraphics[width=\textwidth]{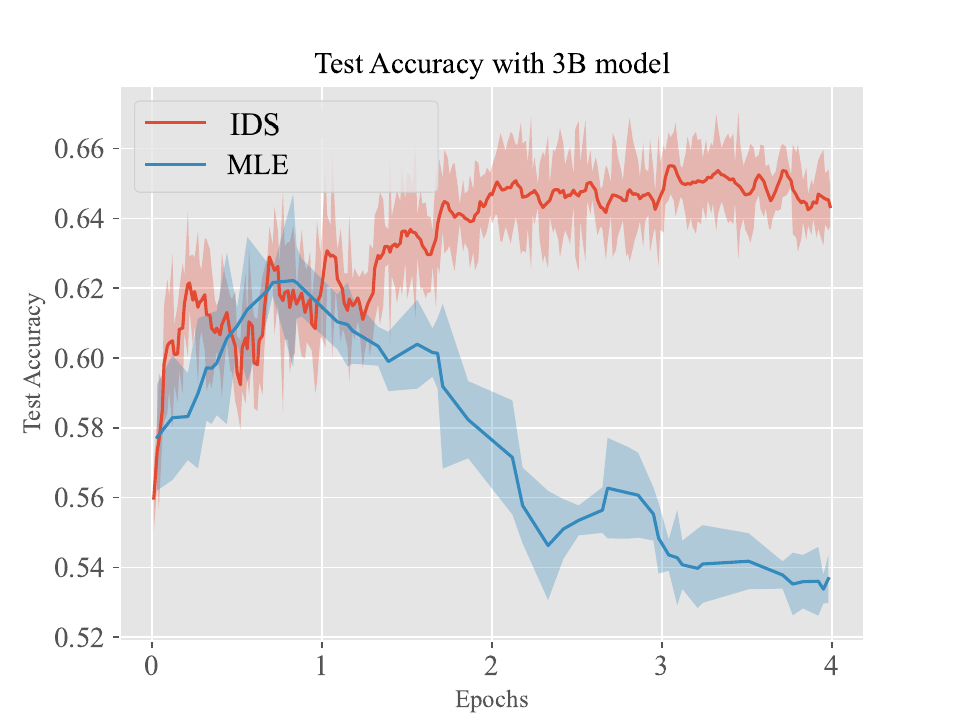}
     \end{subfigure}
     \hfill
     \begin{subfigure}[b]{0.49\linewidth}
         \centering
         \includegraphics[width=\textwidth]{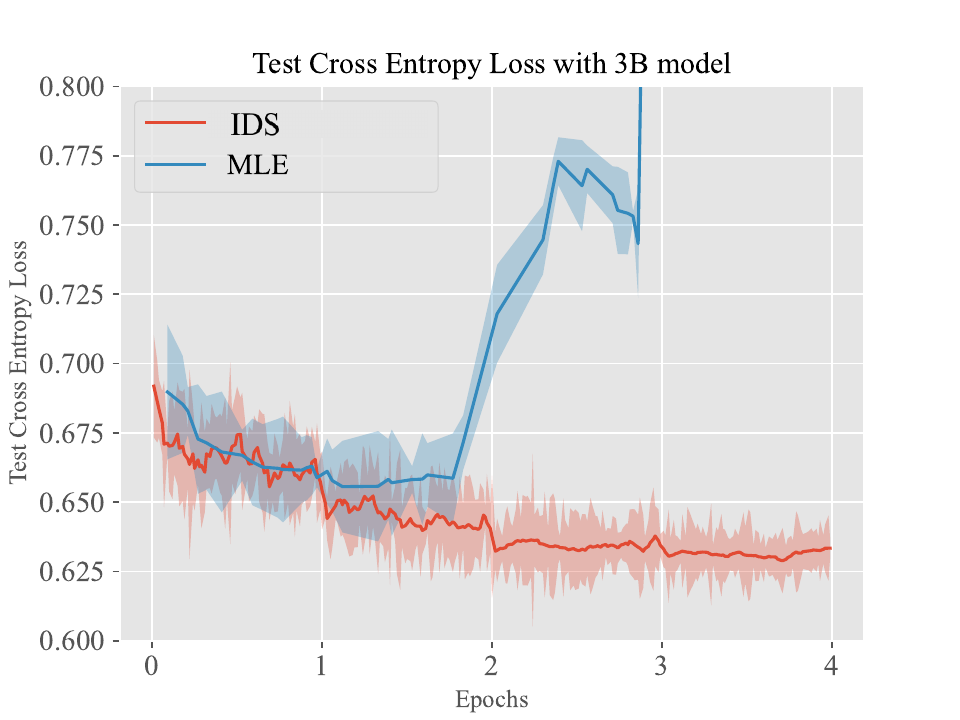}
     \end{subfigure}
     \begin{subfigure}[b]{0.49\linewidth}
         \centering
         \includegraphics[width=\textwidth]{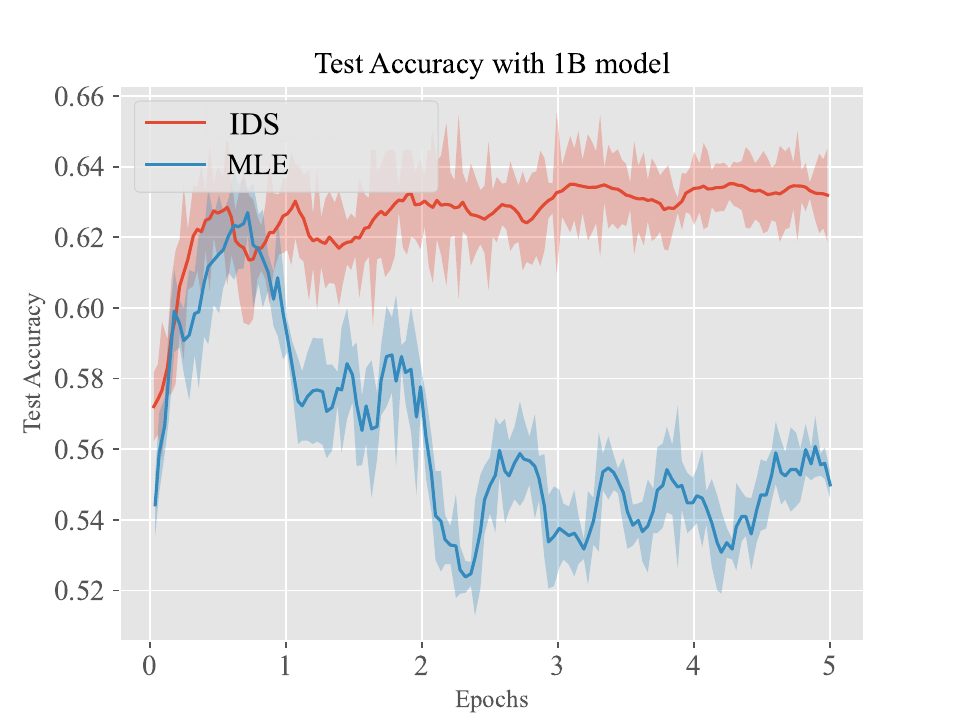}
     \end{subfigure}
     \hfill
     \begin{subfigure}[b]{0.49\linewidth}
         \centering
         \includegraphics[width=\textwidth]{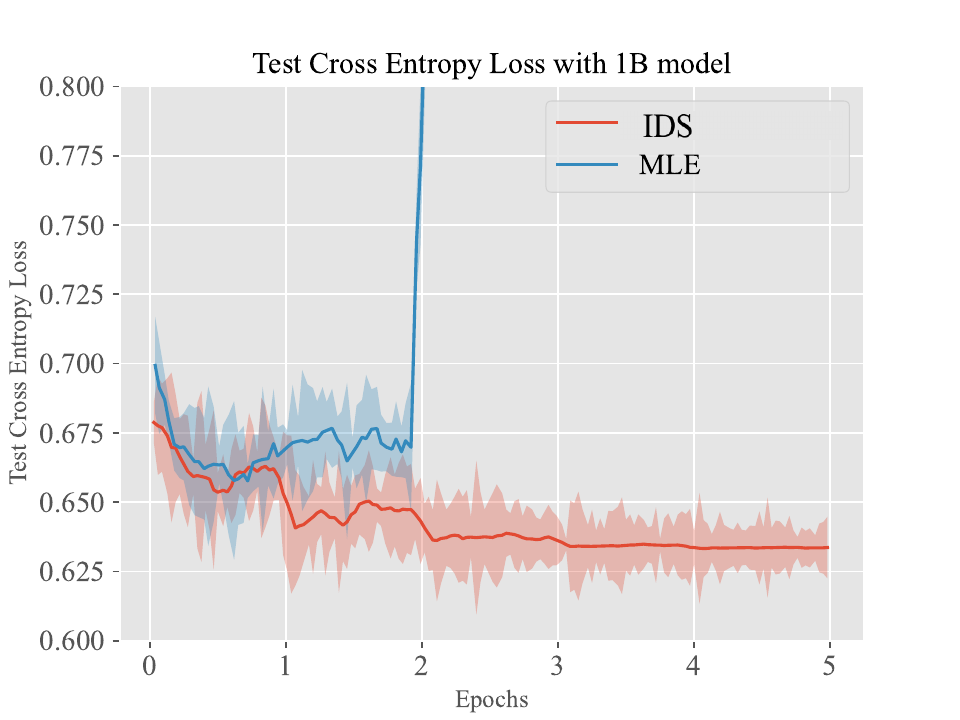}
     \end{subfigure}
      \begin{subfigure}[b]{0.49\linewidth}
         \centering
         \includegraphics[width=\textwidth]{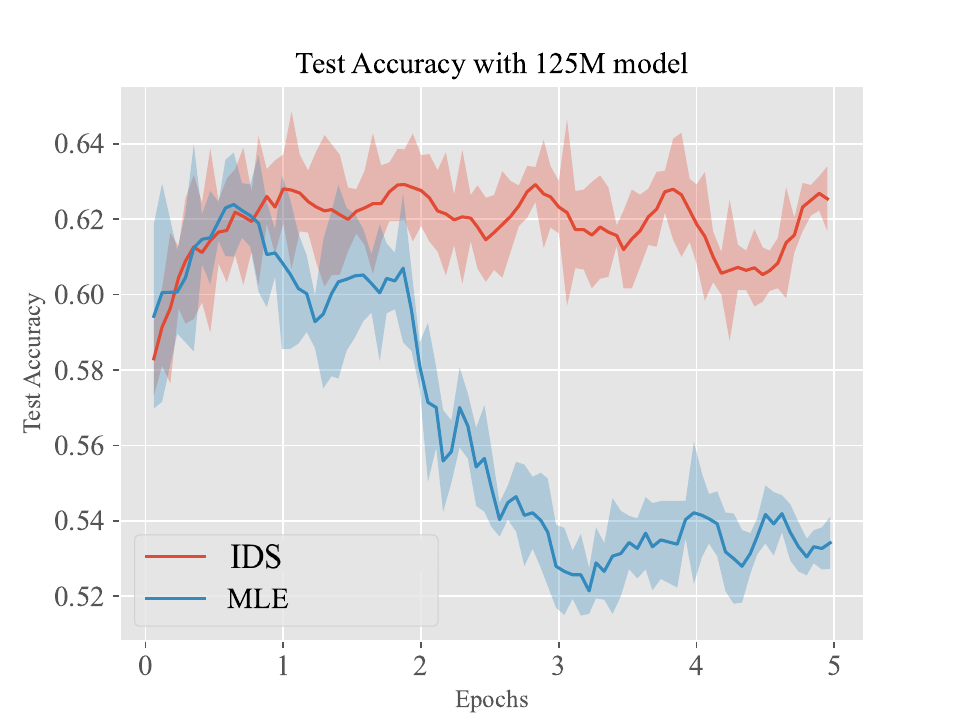}
     \end{subfigure}
     \hfill
     \begin{subfigure}[b]{0.49\linewidth}
         \centering
         \includegraphics[width=\textwidth]{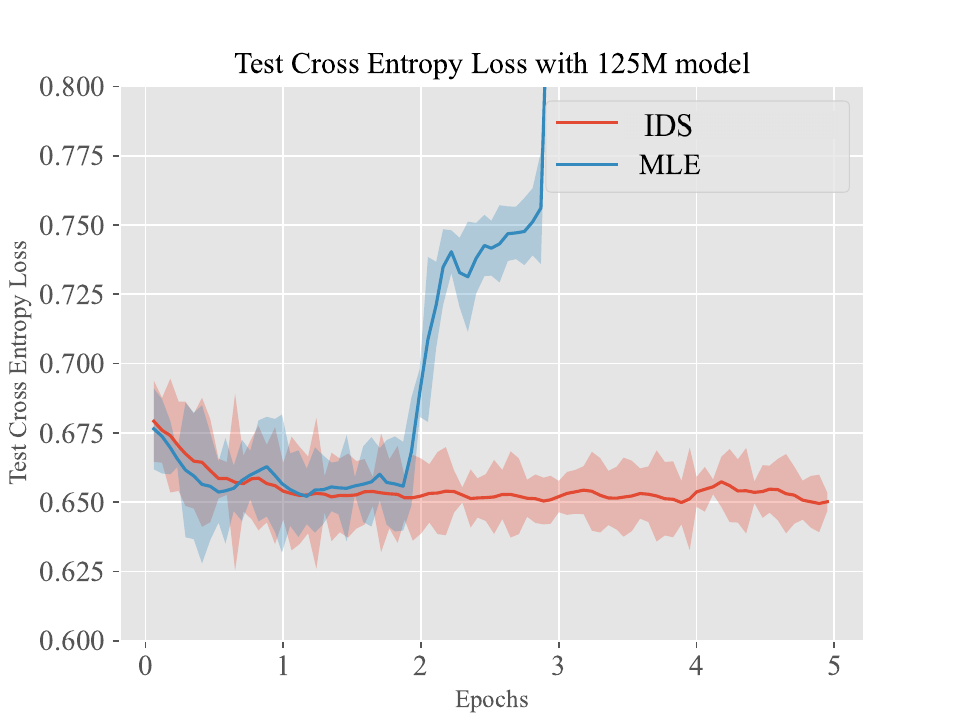}
     \end{subfigure}
    \caption{Comparisons of MLE and IDS  when the reward is parameterized by a neural network. }
    \label{fig:reward}
\end{figure}

 \subsection{Neural Network}\label{sec:exp_nn}
 
We also conduct experiments with neural networks. We use 
 the human-labeled Helpfulness and Harmlessnes (HH) dataset from~\citet{bai2022training}.\footnote{\url{https://huggingface.co/datasets/Dahoas/static-hh}} We take  \texttt{Dahoas/pythia-125M-static-sft}\footnote{\url{https://huggingface.co/Dahoas/pythia-125M-static-sft}} as the policy model with three different reward models of size 125M, 1B and 3B. When training reward model, we take a supervised fine-tuned language model, remove the last layer and replace it with a linear layer.  When fine-tuning the language model, we use the proximal policy optimization (PPO) algorithm~\citep{schulman2017proximal}.

We take a fully-trained 6B reward model  \texttt{Dahoas/gptj-rm-static} trained from the same dataset based on \texttt{EleutherAI/gpt-j-6b} as the ground truth. We use the model to label the comparison samples using the BTL model~\citep{bradley1952rank}. And we train the 125M, 1B and 3B reward model with the new labeled comparison samples.  The reward training results are shown in Figure~\ref{fig:reward}. One can see that the   MLE begins to overfit after 1-2 epochs, while the loss of the IDS algorithm continues to decrease  stably until convergence. 

For both MLE and IDS algortihms, we take the reward with the smallest evaluation loss and optimize a policy against the selected reward model. We compare results for policy learning as shown in  Figure~\ref{fig:policy}. One can see that MLE  suffers from reward overoptimization with few thousand steps, while the ground truth reward continues to grow when using our IDS algorithm. We select step sizes $\alpha =10^{-5}$ and $\beta = 0.7$ for all experiments. We observe that larger model leads to more improvement after one epoch, potentially due to more accurate estimation of the labels.  We provide more details of the experiment along with the experiments on a different dataset, TLDR, in Appendix~\ref{app:exp}. 
\begin{figure}[!htbp]
    \centering
    \includegraphics[width=0.5\textwidth]{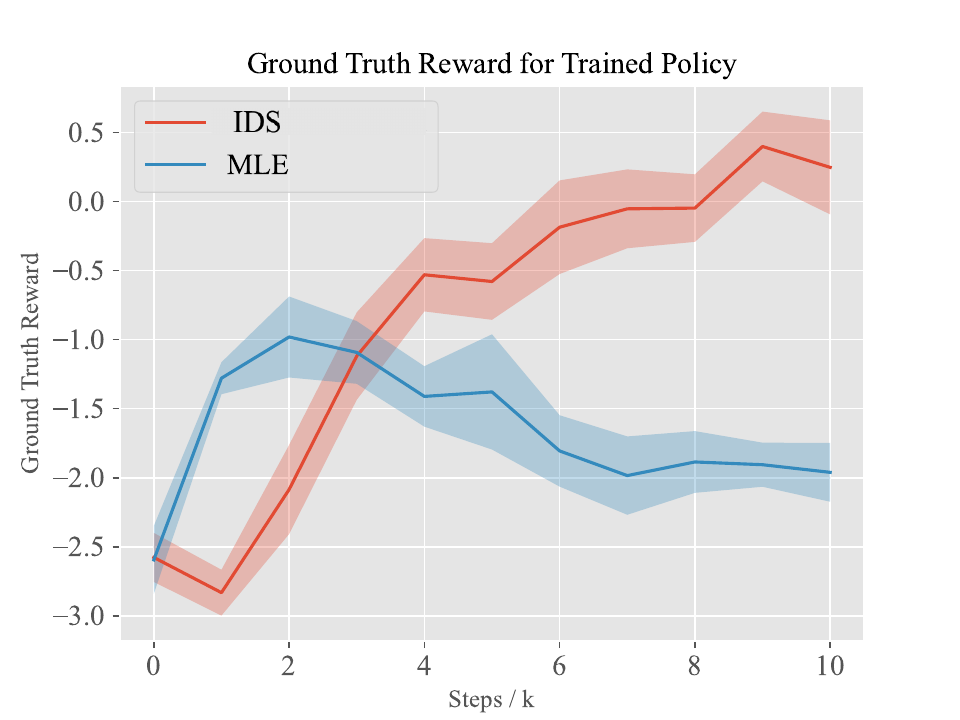}
    \caption{Comparison of MLE and IDS for policy learning}
    \label{fig:policy}
\end{figure}


In the implementation, we find that it is helpful to  restore  the best checkpoint at the end of each epoch. This is due to that an inappropriate label $\{y_i\}_{i=1}^n$ at certain epoch may hurt the performance of the model. To prevent overfitting to the test set, we choose a large validation and test dataset, and we select the best checkpoint according to the smallest loss in the validation set, and plot the loss on the test set. During the whole training procedure including checkpoint restoration, we do not use any of the sample in the test set.

\section{Conclusions
}
We have presented analyses and methodology aimed at resolving the problems of overfitting and overoptimization in reward training for RLHF. We show that our proposed algorithm, IDS, helps mitigate these two issues in both the multi-armed bandit and neural network settings. Note that while we identify the underlying source of reward overfitting and overoptimization as the variance of the human preference data, it is also possible that bias also contributes to these phenomena.
In future work, plan to pursue further formal theoretical analysis of the IDS algorithm, and explore potential applications beyond reward training in the generic domains of classification and prediction.  
 
\section*{Acknowledgements}

The authors would like to thank John Schulman for discussions and suggestions throughout the projects, which initiated the idea of reproducing reward overfitting and overoptimization in the bandit setting, and inspired the idea of combining soft labels during training for better training. The authors would also like to thank Lester Mackey for helpful discussions. The work is supported by NSF Cloudbank and the Mathematical Data Science program of the
Office of Naval Research under grant number N00014-21-1-2840.

\newpage
\bibliography{ref}
\newpage 
\appendix
\section{Extension to Multi-wise Comparison}\label{app:extension}
Here we discuss potential extensions from pairwise comparisons to multi-wise comparison. When there is $M$ ranked responses for each prompt, there are two losses that one can choose from, namely $\mathsf{MLE}_2$ and $\mathsf{MLE}_M$ from~\citet{zhu2023principled}.
\begin{align*}
\hat \theta_{\mathsf{MLE}_2} & \in \argmin_{r}    \mathcal{L}_2(\mathcal{D}, r), \nonumber \\ 
\text{where } \mathcal{L}_2(\mathcal{D}, r) & = - \frac{1}{n}\sum_{i=1}^n \sum_{j=1}^{M } \sum_{k=j+1}^{M } \log \left(\frac{\exp( r (s_i, a_i^{(\sigma_i(j))}) )}{\exp(r (s_i, a_i^{(\sigma_i(j))}) )+\exp(r (s_i, a_i^{(\sigma_i(k))}))}\right) \nonumber \\
\hat \theta_{\mathsf{MLE}_M} & \in \argmin_{r}    \mathcal{L}_M(\mathcal{D}, r), \nonumber\\ 
\text{where } \mathcal{L}_M(\mathcal{D}, r) & = - \frac{1}{n}\sum_{i=1}^n \sum_{j=1}^{M }  \log \left(\frac{\exp(r (s_i, a_i^{(\sigma_i(j))})  )}{\sum_{k=j}^{M} \exp(r (s_i, a_i^{(\sigma_i(k))}) )}\right).\nonumber 
\end{align*}

Here we discuss how to incorporate the iterative data smoothing algorithm for the two losses above. 

The loss $\mathsf{MLE}_2$ splits the $M$-wise comparisons into pairwise comparisons, thus it is straightforward to predict  the new label  $y_i^{j, k}$ for each pair of the comparisons between $j$-th and $k$-th response. The loss used for iterative data smoothing can be written as 
\begin{align*}
    \mathcal{L}_2^{\mathsf{DR}}(\mathcal{D}, r) & = - \frac{1}{n}\sum_{i=1}^n \sum_{j=1}^{M } \sum_{k=j+1}^{M } y_{i}^{\sigma_i(j), \sigma_i(k)}\log \left(\frac{\exp( r (s_i, a_i^{(\sigma_i(j))}) )}{\exp(r (s_i, a_i^{(\sigma_i(j))}) )+\exp(r (s_i, a_i^{(\sigma_i(k))}))}\right) \\ 
    & \qquad + (1-y_{i}^{\sigma_i(j), \sigma_i(k)})\log \left(\frac{\exp( r (s_i, a_i^{(\sigma_i(k))}) )}{\exp(r (s_i, a_i^{(\sigma_i(j))}) )+\exp(r (s_i, a_i^{(\sigma_i(k))}))}\right).
\end{align*} 
\begin{align*}
    y_{i, t+1}^{j, k} = (1-\beta)\cdot y_{i, t}^{j, k}  + \beta\cdot  \frac{\exp( r_{\theta_{t+1}} (s_i, a_i^{j}) )}{\exp(r_{\theta_{t+1}} (s_i, a_i^{j}) )+\exp(r_{\theta_{t+1}} (s_i, a_i^{k}))}.
\end{align*}

On the other hand, adapting the loss $\mathsf{MLE}_M$ for iterative data smoothing requires more efforts since it requires changing the ranking labels to soft labels. The design of $\mathsf{MLE}_M$ decomposes the probability of the observed ranking to the product of the probability that each response is the most preferred one among the rest of the responses.  One of the options is to directly change the labels for the current rankings by the following update rules:
\begin{align*}
    \mathcal{L}_M(\mathcal{D}, r) & = - \frac{1}{n}\sum_{i=1}^n \sum_{j=1}^{M } y_i^{\sigma_i(j)} \log \left(\frac{\exp(r (s_i, a_i^{(\sigma_i(j))})  )}{\sum_{k=j}^{M} \exp(r (s_i, a_i^{(\sigma_i(k))}) )}\right).  
\end{align*}
And the update rule for the labels $y_i$ is
\begin{align*}
    y_{i, t+1}^{\sigma_i(j)} = (1-\beta)\cdot y_{i, t}^{\sigma_i(j)}  + \beta\cdot  \frac{\exp( r_{\theta_{t+1}} (s_i, a_i^{(\sigma_i(j))}) )}{\sum_{k=j}^{M} \exp(r_{\theta_{t+1}} (s_i, a_i^{(\sigma_i(k))}) )}.
\end{align*}
However, the above update method does not directly reduce to the the case of pairwise comparisons when setting $M=2$. In order to recover the pairwise loss, one needs to consider all possible rankings and get the soft labels for all the rankings. The loss will become
\begin{align*}
    \mathcal{L}_M'(\mathcal{D}, r) & = - \frac{1}{n}\sum_{i=1}^n \sum_{\sigma\in\Pi(M)}\sum_{j=1}^{M } y_i^{j,\sigma} \log \left(\frac{\exp(r (s_i, a_i^{(\sigma(j))})  )}{\sum_{k=j}^{M} \exp(r (s_i, a_i^{(\sigma(k))}) )}\right).  
\end{align*}
Here $\Pi(M)$ is the set of all permutations of the $M$ elements. And the label is initialized as $y_{i, 0}^{j, \sigma} = 1$ if $\sigma  = \sigma_i$, and $0$ otherwise. 
And the update rule for the labels $y_i$ is
\begin{align*}
    y_{i, t+1}^{j, \sigma} = (1-\beta)\cdot y_{i, t}^{j, \sigma}  + \beta\cdot  \frac{\exp( r_{\theta_{t+1}} (s_i, a_i^{(\sigma(j))}) )}{\sum_{k=j}^{M} \exp(r_{\theta_{t+1}} (s_i, a_i^{(\sigma(k))}) )}.
\end{align*}
The loss is consistent with the pairwise cross entropy loss when $M=2$. However, it requires enumerating over all possible permutations, which are not very efficient when $M$ is large. It requires more study to decide which loss is more appropriate for $M$-wise iterative data smoothing.

\section{An Alternative Formulation of Iterative Data Smoothing}\label{app:alternative}

Besides the formulation shown in Algorithm~\ref{alg:refine}, we also propose an alternative formulation that directly multiplies a confidence $c_i$ in front of the original loss for each sample, as is shown in Algorithm~\ref{alg:refine_v2}. We note here that although the algorithm has better asymptotic convergence result, its  performance in practice is not as good as Algorithm~\ref{alg:refine}.

\begin{algorithm}[!htbp]
\caption{Iterative Data Smoothing V2 ($\mathcal{D}, \theta_0, \alpha, \beta$)}
\label{alg:refine_v2}
\begin{algorithmic}
\STATE \textbf{Input:} The pairwise comparison dataset $\mathcal{D} = \{a_i, a_i', y_i\}_{i=1}^n$. A parameterized reward model family $\{r_\theta: \mathcal{A} \mapsto \mathbb{R} \mid \theta\in\Theta\}$ with initialization $\theta_0\in\Theta$. Two step sizes $\alpha, \beta$. An empirical loss function $$\mathcal{L}_\theta(\{c_i\}, \mathcal{D})   =- \frac{1}{n}
\sum_{i=1}^n \max(2c_i-1, 0)\cdot \left( y_i\cdot \log\left(\frac{\exp(  r_\theta(a_i))}{\exp( r_\theta(a_i)) + \exp( r_\theta(a_i'))}\right) + (1-y_i)\cdot \log\left(\frac{\exp(r_\theta(a_i'))}{\exp(r_\theta(a_i)) + \exp( r_\theta(a_i'))}\right)\right)$$
\STATE Initialize $t=0$ and $c_{i, 0} = 1, \forall i$.
\WHILE{$r_{\theta_t}$ does not converge}{ \STATE \begin{align*}
        \theta_{t+1} & \gets \theta_t - \alpha \cdot \nabla \mathcal{L}_\theta(\{c_{i, t}\}, \mathcal{D})  \\
        c_{i, t+1} & \gets (1-\beta) \cdot c_{i, t} + \beta \cdot \frac{\exp(  r_{\theta_{t+1}}(a_i))}{\exp( r_{\theta_{t+1}}(a_i)) + \exp( r_{\theta_{t+1}}(a_i'))} \\
        t & \gets t+1
     \end{align*}
}
\ENDWHILE 
\STATE \textbf{Return:} $r_{\theta_{t}}$
\end{algorithmic}
\end{algorithm}
We multiply a $\max(2c_i-1, 0)$ in front of the loss for each sample as an approximation of how confident the current model predicts the preference label. When the reward is approximately similar, the coefficient goes to $0$, putting less weights on those samples.
Below we show that asymptotically, the new iterative data smoothing V2 algorithm is better at preserving the preference distribution compared with the original version.
\begin{theorem}\label{thm:dr_asymp}
Consider the multi-armed bandit problem with the number of samples going to infinity and a fixed sampling distribution $\mu$. Assume that $\mu(a, a')>0$ for any $a, a'>0$. Then we have
\begin{itemize}
    \item Any stationary point for Algorithm~\ref{alg:refine} satisfies $\forall a, a', \hat r(a) = \hat r(a')$;
    \item There is one stationary point for Algorithm~\ref{alg:refine_v2} that satisfies
    \begin{align*}
        \forall a, a',  \hat r(a) - \hat r(a') =  r^\star(a) - r^\star(a'). 
    \end{align*}
\end{itemize} 
\end{theorem} 
\begin{proof}
The stationary points for Algorithm~\ref{alg:refine} and~\ref{alg:refine_v2} are the points where the gradients equal $0$. For Algorithm~\ref{alg:refine}, this is equivalent to $\hat y = \frac{\exp(  \hat r(a))}{\exp( \hat r(a)) + \exp( \hat r(a'))}$, and 
\begin{align*}
     & ({\mu(a\succ a')} \cdot \hat y +  \mu(a \prec a') \cdot (1-\hat y) ) \cdot  \frac{\exp(  \hat r(a'))}{\exp( \hat r(a')) + \exp( \hat r(a'))}  \\
    = &({ \mu(a\prec a')} \cdot \hat y +  \mu(a \succ a') \cdot (1-\hat y) ) \cdot  \frac{\exp(\hat   r(a))}{\exp( \hat r(a)) + \exp( \hat r(a'))}.
\end{align*}
Here $\hat y$ can be different for different $(a, a')$.
Solving the above equation gives that the only stationary point is $\hat y = 1/2$ and $\hat r(a) - \hat r(a') = 0$.

On the other hand, for Algorithm~\ref{alg:refine_v2}, the stationary point condition is equivalent to $\hat c(a,a') = \frac{\exp(  \hat r(a))}{\exp( \hat r(a)) + \exp( \hat r(a'))}$, and 
\begin{align*}
     &\sum_{a'} \max(2\hat c(a,a')-1, 0)\cdot\left({\mu(a\succ a')} \cdot  \frac{\exp(  \hat r(a'))}{\exp( \hat r(a')) + \exp( \hat r(a'))} -  { \mu(a\prec a')} \cdot  \frac{\exp(\hat   r(a))}{\exp( \hat r(a)) + \exp( \hat r(a'))}\right) = 0.
\end{align*}
Thus one can verify that $\hat r(a) - \hat r(a') = r^\star(a) - r^\star(a')$ satisfies the stationary condition.  
This proves the result.     
\end{proof}
\begin{remark}
    Although the asymptotic stationary points of Algorithm~\ref{alg:refine} do  not contain the ground truth, the two-scale analysis discussed in Section~\ref{sec:benefits_dr} shows that when one of the step size is much larger than the other such that  one of the updates in $\hat y$ or $\hat r$ is slower (and thus does not hit the stationary point),  the reward still converges to the ground truth for those sufficiently observed arms. However, preliminary experiments on Algorithm~\ref{alg:refine_v2} show that the result is worse than that of Algorithm~\ref{alg:refine}, and also suffer from reward overfitting. This together with the failure of MLE  may suggest that asymptotic result does not reflect the practical performance with smaller sample size compared with number of parameters.
\end{remark} 
\begin{remark}
    The condition of $\mu(a, a')>0$ can be relaxed to that the comparison graph induced by the Laplace matrix $L$ is connected, since the reward is identifiable in this case~\citep{shah2015estimation}. 
\end{remark}
\section{Additional Experiments}\label{app:exp}

The hyper-parameters for the neural network experiments are listed in table \ref{tab:hyp_stackx}.
\begin{table}[H]
    \centering
    \begin{tabular}{c|c|c}
      Model & Parameter&Value \\
    \hline 
         \multirow{4}{*}{Reward model}&  learning rate $\alpha$ & $10^{-5}$ \\
        & label update parameter $\beta$ & 0.7 \\
       & batch size & 128 \\ 
       & eval \& save steps & 100 \\
       \hline 
         \multirow{10}{*}{Policy model}&  max sequence length& 1024   \\
         &max output length& 500 \\
         &generation temperature & 1.0 \\
         & batch size  & 64 \\
         & fixed KL   coefficient & 0.001   \\
         & number of rollouts & 128 \\
         & PPO epochs & 4 \\
         & value coefficient & 0.5 \\
         & GAE  coefficient $\lambda$ & 0.95 \\
         & discount factor & 1.0 \\ 
         & clip range & 2 \\
    \end{tabular}
    \caption{Hyper-parameters for the neural network experiments}
    \label{tab:hyp_stackx}
\end{table}

We also include additional experiments on a different dataset, TLDR\footnote{\url{https://huggingface.co/datasets/CarperAI/openai_summarize_comparisons}}, in Figure~\ref{fig:tldr_reward} and \ref{fig:tldr_policy} of this section. The settings follow the same as HH in Section~\ref{sec:exp_nn}.  One can see that in the case of TLDR, the test accuracy does not drop significantly like HH. However, even with small difference in loss and the accuracy, the resulting policy reward difference is still significant. 

\begin{figure}[!htbp]
     \centering
     \begin{subfigure}[b]{0.49\linewidth}
         \centering
         \includegraphics[width=\textwidth]{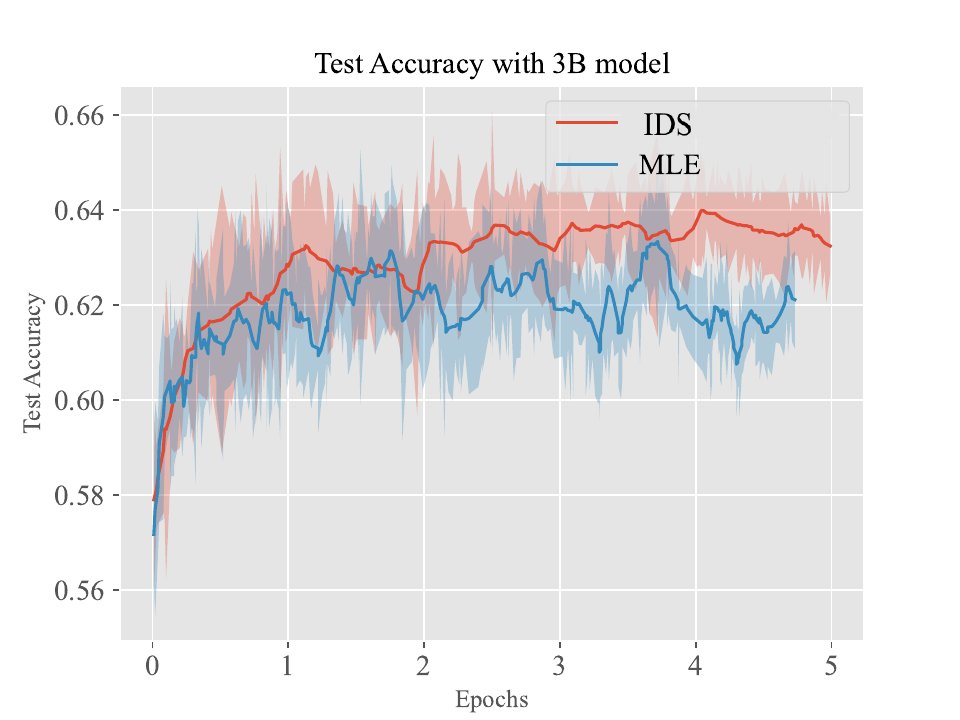}
     \end{subfigure}
     \hfill
     \begin{subfigure}[b]{0.49\linewidth}
         \centering
         \includegraphics[width=\textwidth]{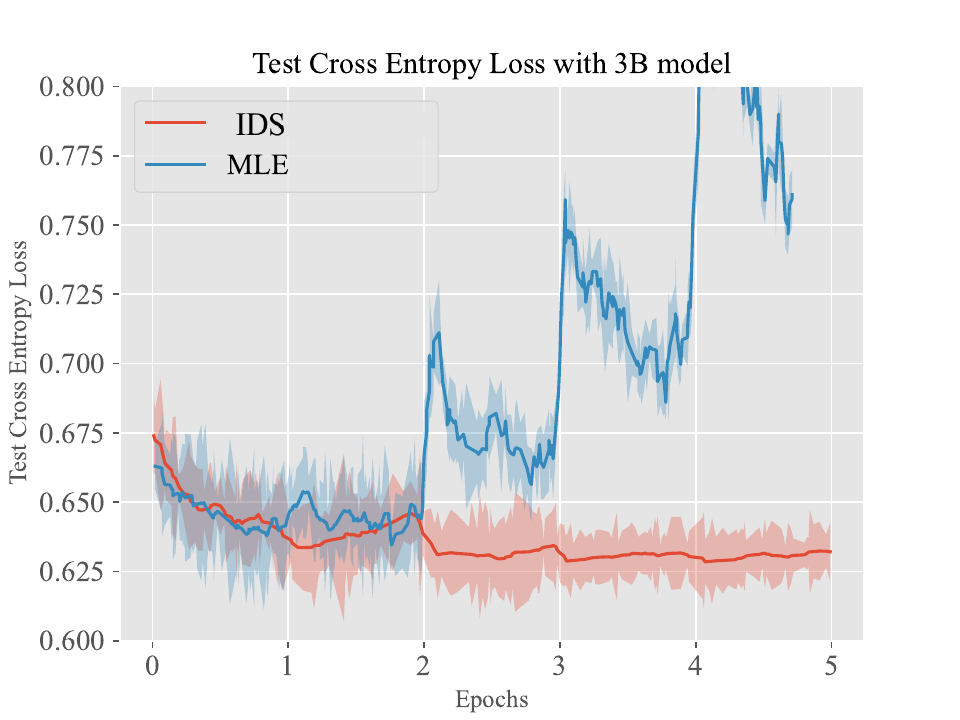}
     \end{subfigure}
     \begin{subfigure}[b]{0.49\linewidth}
         \centering
         \includegraphics[width=\textwidth]{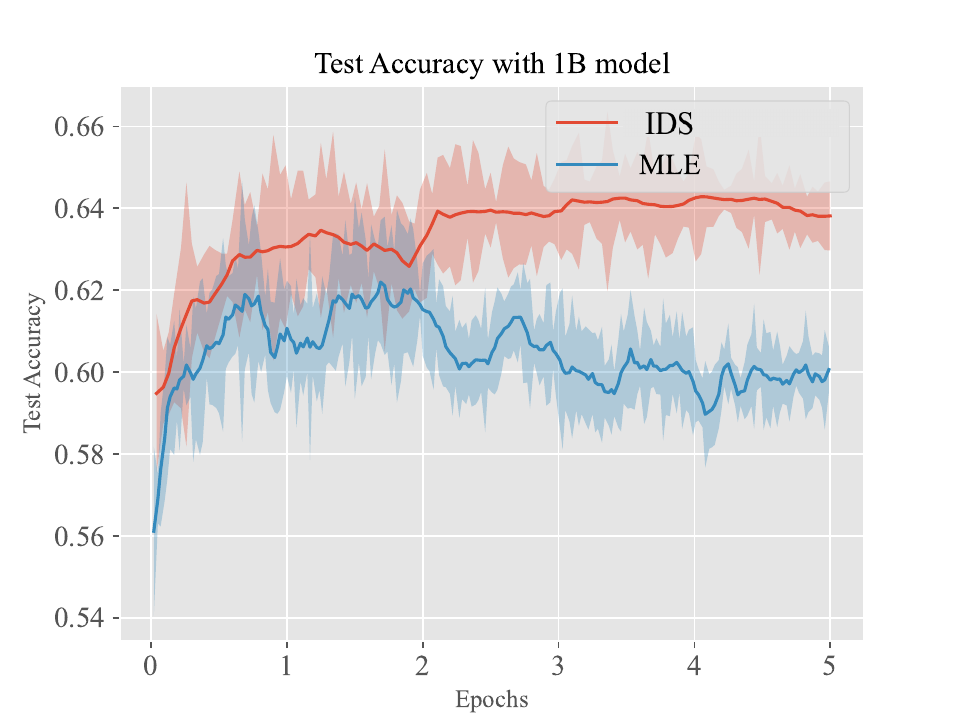}
     \end{subfigure}
     \hfill
     \begin{subfigure}[b]{0.49\linewidth}
         \centering
         \includegraphics[width=\textwidth]{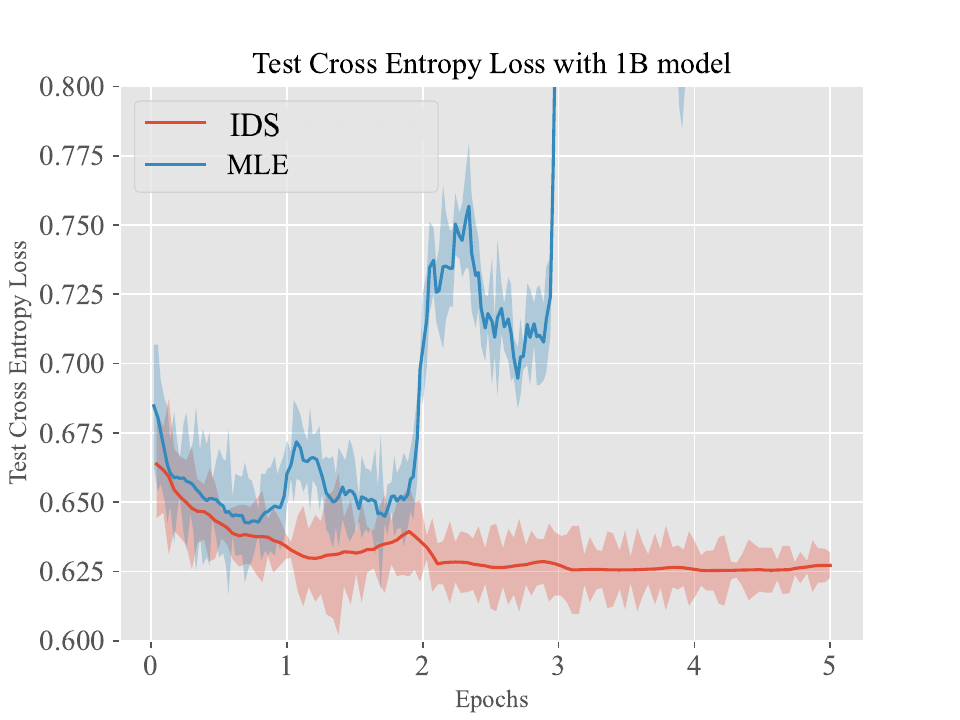}
     \end{subfigure}
      \begin{subfigure}[b]{0.49\linewidth}
         \centering
         \includegraphics[width=\textwidth]{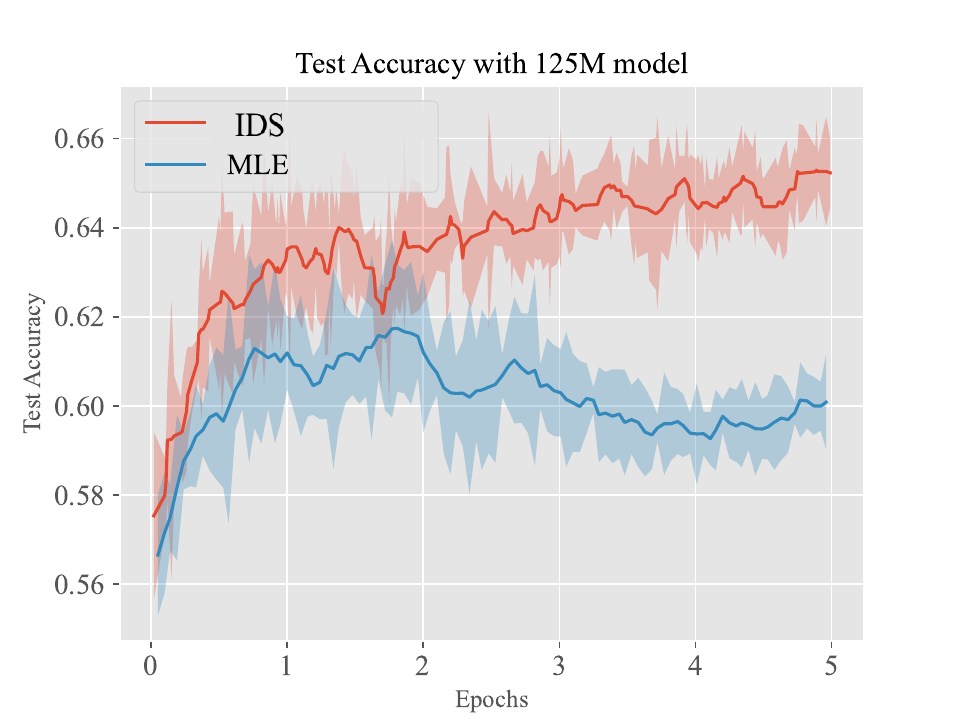}
     \end{subfigure}
     \hfill
     \begin{subfigure}[b]{0.49\linewidth}
         \centering
         \includegraphics[width=\textwidth]{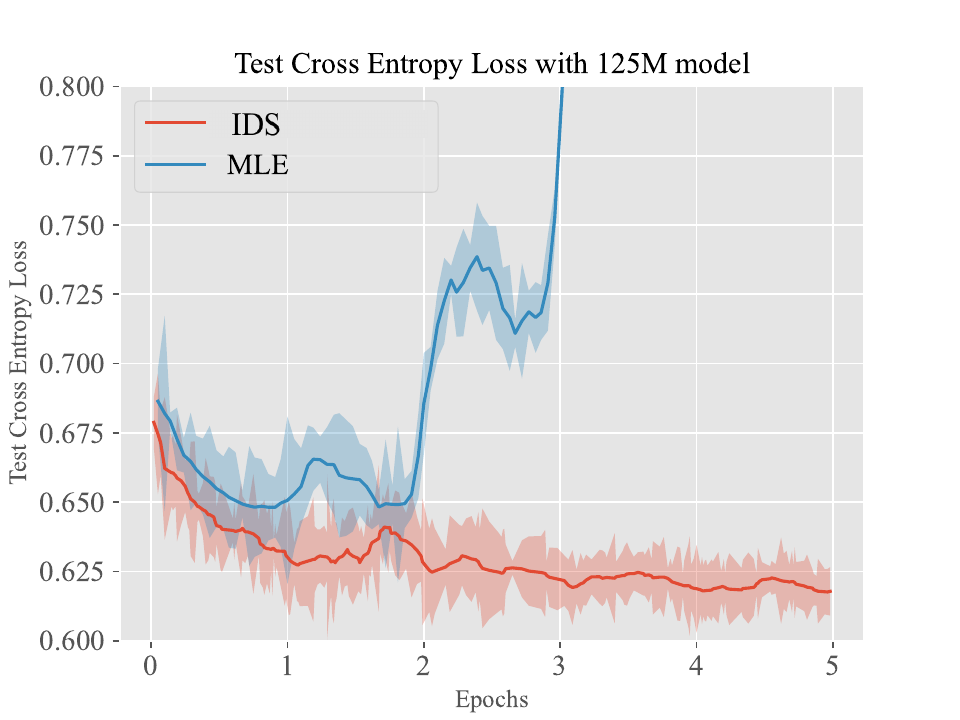}
     \end{subfigure}
    \caption{Comparisons of MLE and Iterative Data Smoothing  when the reward is parameterized by a neural network. }
    \label{fig:tldr_reward}
\end{figure}
\begin{figure}[!htbp]
    \centering
    \includegraphics[width=0.5\textwidth]{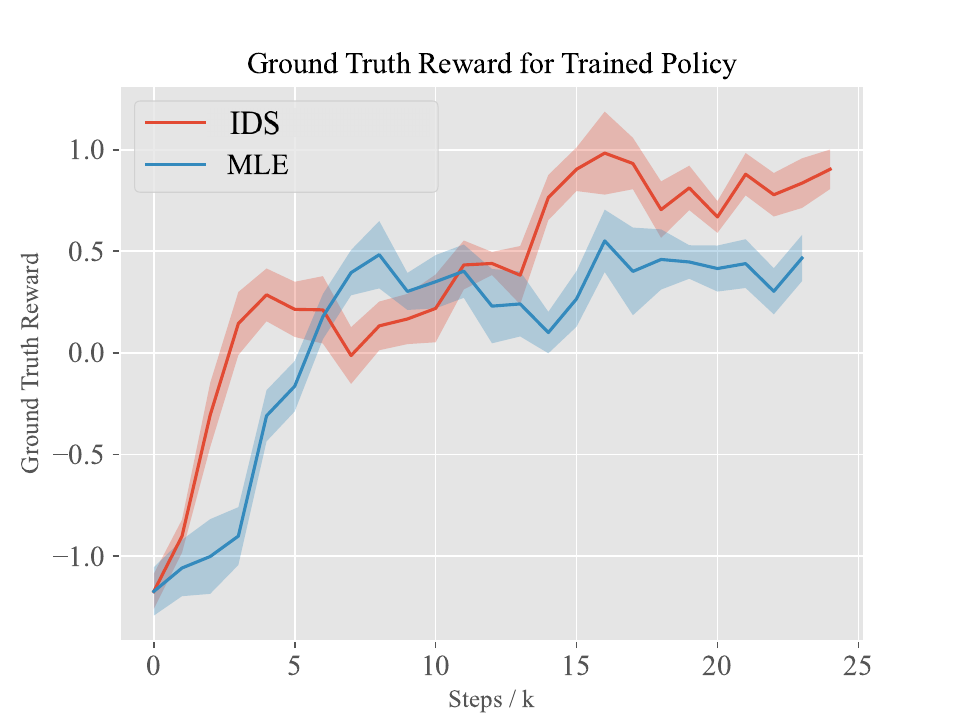}
    \caption{Comparison of MLE and Iterative Data Smoothing for policy learning.}
    \label{fig:tldr_policy}
\end{figure}

\section{Proof of Theorem~\ref{thm:mle}}\label{proof:mle_consistency}
\begin{proof}
Let $\mathbb{P}_r(a, a', c) = \mathbb{P}_r(a\succ a')$ if $c=1$, and $\mathbb{P}_r(a'\succ a) $ if $c=0$ be the density function of the observations.
According to Theorem 6.1.3. of \citet{hogg2013introduction}, it suffices to verify the following conditions for the consistency of MLE:
\begin{itemize}
    \item The CDFs are distinct, i.e. $\mathbb{P}_r(a, a', c) = \mathbb{P}_{r'}(a, a', c) $ almost everywhere implies that $ r = r'$. This is true since the distribution is supported on discrete space, and the equality implies that $r(i)-r(j) = r(i)' - r(j)'$ for any $i, j$, and  $r(M) = r'(M) = 1$.
    \item The PDFs have common support for all $r$.  This is true since the probability is positive for any $a, a', c$.
    \item The point $r^\star$ is an interior point in $\mathbb{R}^K$. This is  true by definition, since any open ball of radius $\epsilon$ around $r^\star$ is a subset of the space.

\end{itemize}
\end{proof}
\section{Proof of Theorem~\ref{thm:lower_mle}}\label{proof:lower_mle}
\begin{proof}
The construction is in similar spirit to~\cite{rashidinejad2021bridging} and~\cite{zhu2023principled}.  Consider a  bandit problem where $r^\star(a) = \mathbbm{1}(a=1)$. For any fixed $n$, we set $\mu(1, 2) = 1-1/n$,  $\mu(1, 3) = 1/n$. 

In this hard instance, there is constant probability that arm $3$ is only compared with arm $1$ once. Concretely, we  have
\begin{align*}
    \mathbb{P}(n(1,3)=1) & = n\cdot (1-\mu(1, 3))^{n-1}\cdot \mu(1, 3)= \left(1-1/n\right)^{n-1}.
\end{align*}
When $n \geq 500$, we  have $ \mathbb{P}(n(1,3)=1) \geq 0.36$. 
Under this case, we know that arm $3$ is preferred with probability at least $\exp(r(3)) / (\exp(r(1))+\exp(r(3))) >0.26$. When there is only one comparison between arm $1$ and $3$, and arm $3$ is preferred, the MLE assigns $r(3)$ as infinity.  
Even when the reward for arm $1$ is estimated perfectly, this leads to a population cross-entropy loss arbitrarily large.
 
\end{proof}

\section{Proof of Corollary~\ref{thm:fail_mle}}\label{proof:fail_mle}
\begin{proof}
    
The proof  follows immediately from  Theorem~\ref{thm:lower_mle}. Under the same construction, we know that $\hat r_{\mathsf{MLE}}(3) = +\infty$ with probability at least $0.09$. Thus, the sub-optimality of the resulting optimal policy is at least $1$.
\end{proof}
\section{Proof of Theorem~\ref{thm:gd}}\label{proof:gd} 
\begin{proof}
Let $\hat r_i$ be the reward for the $i$-th arm, and $\hat r = [\hat r_1, \hat r_2,\cdots, \hat r_K]$ as the vector for the reward. 
One can calculate the  gradient of the reward as 
\begin{align*}
\nabla_{\hat r_i} \mathcal{L}_{\mathsf{CE}}(\mathcal{D}, \hat r) & =- \frac{1}{n}
\sum_{i=1}^n \nabla_{\hat r_i} \Bigg( y_i\log\left(\frac{\exp(\hat r_{a_i})}{\exp(\hat r_{a_i}) + \exp(\hat r_{a_i'})}\right)   + (1-y_i))\log\left(\frac{\exp(\hat r_{a_i'})}{\exp(\hat r_{a_i}) + \exp(\hat r_{a_i'})}\right)\Bigg) \nonumber \\ 
 & =- \frac{1}{n}
\sum_{i=1}^n  \Bigg(\frac{y_i\exp(\hat r_{a_i'})}{\exp(\hat r_{a_i})) + \exp(\hat r_{a_i'})} - \frac{(1-y_i)\exp(\hat r_{a_i})}{\exp(\hat r_{a_i}) + \exp(\hat r_{a_i'})} \Bigg) \nonumber \\ 
& = -\frac{1}{2}\cdot (n_+(i) - n_-(i)).
\end{align*}
Here the last equality is due to that  all the reward   is initialized at the same value. And $n_+(i)$ (or $n_-(i)$) refers the total number of winning (or losing) of arm $i$ in the observations.

We assume all the reward is initialized at $0$ without loss of generality. After one step gradient, we have
\begin{align*}
    \hat r(i) = \alpha (n_+(i) - n_-(i)).
\end{align*}
This proves the result. 

\end{proof}
\section{Proof of Theorem~\ref{thm:differential}}\label{proof:differential}

\begin{proof}
    From the differential equations in (\ref{eq:diff}), we know that
    \begin{align*}
        \frac{\dot{y}(t)}{y(t)} \geq -\beta.
    \end{align*}
    Taking integration on both sides give us
    \begin{align*}
        y(t) \geq \exp(-\beta t) \geq \exp(-\epsilon).
    \end{align*}
    Now we set a Lyapunov function $V(t) =  \left(\frac{\exp(d(t))}{1+\exp(d(t))} - \mu\right)^2$. We know that
    \begin{align*}
        \dot{V}(t) & = 2\left(\frac{\exp(d(t))}{1+\exp(d(t))} - \mu\right)\cdot \frac{\exp(d(t))}{(1+\exp(d(t)))^2}\cdot \dot{d}(t) \\ 
        & = 2\alpha n \left(\frac{\exp(d(t))}{1+\exp(d(t))} - \mu\right)\cdot \frac{\exp(d(t))}{(1+\exp(d(t)))^2} \\
        & \qquad \cdot \left((\mu\cdot y(t) + (1-\mu)\cdot (1-y(t)))\cdot \frac{1}{1+\exp(d(t))} - ((1-\mu)\cdot y(t) + \mu\cdot (1-y(t)))\cdot \frac{\exp(d(t))}{1+\exp(d(t))}\right) \\
        & = 2\alpha n \left(\frac{\exp(d(t))}{1+\exp(d(t))} - \mu\right)\cdot \frac{\exp(d(t))}{(1+\exp(d(t)))^2} \cdot \left((2\mu-1)\cdot y(t) + 1-\mu- \frac{\exp(d(t))}{1+\exp(d(t))}\right) \\
        & = -2\alpha n \cdot  \frac{\exp(d(t))}{(1+\exp(d(t)))^2} \cdot \left(\frac{\exp(d(t))}{1+\exp(d(t))} - \mu\right)^2  \\ 
        & \qquad + 2\alpha n \cdot  \frac{\exp(d(t))}{(1+\exp(d(t)))^2} \cdot \left(\frac{\exp(d(t))}{1+\exp(d(t))} - \mu\right) \cdot (2\mu-1)\cdot (y(t)-1) \\
        & \stackrel{(i)}{\leq} 2\alpha n \cdot  \frac{\exp(d(t))}{(1+\exp(d(t)))^2} \cdot \left(-\left(\frac{\exp(d(t))}{1+\exp(d(t))} - \mu\right)^2 + 1-\exp(-\epsilon)\right) \\ 
        & = 2\alpha n \cdot  \frac{\exp(d(t))}{(1+\exp(d(t)))^2} \cdot \left(-V(t) + 1-\exp(-\epsilon)\right).
    \end{align*}
    Here (i) uses the fact that $y(t), \mu, \frac{\exp(d(t))}{1+\exp(d(t))}\in [0, 1]$.
    Now consider two scenarios. The first is that for any time $t\in[0, T]$, one always has $V(t) \geq 2(1-\exp(-\epsilon))$. In this case, we know that
    \begin{align}
         \dot{V}(t) & \leq 2\alpha n \cdot  \frac{\exp(d(t))}{(1+\exp(d(t)))^2} \cdot \left(-V(t) + 1-\exp(-\epsilon)\right) \nonumber \\
         & \leq -\alpha n  \cdot  \frac{\exp(d(t))}{(1+\exp(d(t)))^2} \cdot V(t) \nonumber \\
         & \leq 0. \label{eq:non_increasing}
    \end{align}
This shows that $V(t)$ is a non-increasing function.
Without loss of generality, assume that $\mu\geq 1/2$. We know that
\begin{align}
  V(t)\leq V(t'), \forall t>t'.\label{eq:inequality_V}
\end{align}
Now we prove that there must be $\frac{\exp(d(t))}{1+\exp(d(t))} \leq \mu$. If one can find some $t_0$ such that $\frac{\exp(d(t))}{1+\exp(d(t))} > \mu$, by the continuity of $\frac{\exp(d(t))}{1+\exp(d(t))}$ and the fact that $\frac{\exp(d(0))}{1+\exp(d(0))}=1/2$, one can find some $t_1<t_0$ such that  $\frac{\exp(d(t_1))}{1+\exp(d(t_1))} = \mu$. 
This gives that
\begin{align*}
   V(t_1) = 0 < V(t_0),
\end{align*}
which contradicts Equation (\ref{eq:inequality_V}). Thus we know that $\frac{\exp(d(t))}{1+\exp(d(t))} \leq \mu$ holds for any $t$. Furthermore, since we know that $V(t)$ is non-increasing, we know that $\frac{\exp(d(t))}{1+\exp(d(t))} \geq 1/2$. This also implies that
\begin{align*}
    \frac{\exp(d(t))}{(1+\exp(d(t)))^2} \geq \mu(1-\mu ).
\end{align*}
Similarly, we can prove the same condition holds when $\mu<1/2$. 
Thus we have
\begin{align*}
    \frac{\dot{V}(t)}{V(t)} \leq -\mu(1-\mu )\alpha n.
\end{align*}
By integrating over $t$ on both sides, we have
\begin{align*}
    V(t) \leq \exp(-\mu(1-\mu)\alpha n t)\cdot V(0)\leq  \exp(-\mu(1-\mu)\alpha n t).
\end{align*}
Here the last inequality uses the fact that $V(0)\in[0, 1]$.

On the other hand, assume that at some time point $t_0\in[0, T]$, we have $V(t_0) < 2(1-\exp(-\epsilon))$.  When  $V(T) > 2(1-\exp(-\epsilon))$, by the continuity of the function $V(\cdot)$, we know that there exists some $t_1$ such that $V(t_1) = 2(1-\exp(-\epsilon))$, and  for any $t\in[t_1, T]$, $V(t) \geq 2(1-\exp(-\epsilon))$. From Equation (\ref{eq:non_increasing}), we know that in the regime of $t\in[t_1, T]$, $V(t)$ is non-increasing. This contradicts with the fact that $V(T)>V(t_1)$. Thus we know that $$V(T) \leq 2(1-\exp(-\epsilon)).$$
 
\end{proof}

\end{document}